%% file: 0_main_arxiv.tex
\documentclass[11pt]{article}
\usepackage[utf8]{inputenc} 
\usepackage[T1]{fontenc}    

\input{tex/command.tex}

\usepackage[numbers]{natbib}
\usepackage{geometry}
\usepackage{xr}

\title{A Finite-Sample Analysis of Quantile Temporal-Difference Learning}
\author{
Zijie Cheng\thanks{School of Mathematical Sciences, Peking University; email:
\texttt{xmwsbsj@stu.pku.edu.cn}.}
\and Xiang Li\thanks{Perelman School of Medicine, University of Pennsylvania;
email: \texttt{lx10077@upenn.edu}.}
\and Yang Peng\thanks{Yau Mathematical Sciences Center, Tsinghua University;
email: \texttt{yang-peng@mail.tsinghua.edu.cn}.}
\and Zhihua Zhang\thanks{School of Mathematical Sciences, Peking University;
email: \texttt{zhzhang@math.pku.edu.cn}.}}
\date{}

\begin{document}
\maketitle

\begin{abstract}
\input{0_abstract}
\end{abstract}

\noindent\textbf{Keywords:} distributional reinforcement learning; quantile
temporal-difference learning; finite-sample convergence; stochastic
approximation.

\input{1_introduction}
\input{2_problem_setup}
\input{3_main_results}
\input{4_proof_sketch}
\input{5_simulations}
\input{6_conclusion}

\bibliography{ref}
\bibliographystyle{abbrvnat}

\clearpage
\appendix
\input{tex/S_supplementary_related_work}

\section{Additional Preliminaries}
\label{Appendix_additional_preliminaries}
\input{tex/C_additional_preliminaries}
\section{Proofs for Sections~\ref{Section:analysis} and~\ref{Section:proof_sketch}}
\label{Appendix_omitted_proofs}
\input{tex/A_omitted_proofs}

\section{Technical Lemmas}
\label{Appendix_technical_lemmas}
\input{tex/B_technical_lemmas}

\input{tex/D_additional_simulations}

\end{document}

%% file: tex/command.tex
\usepackage[T1]{fontenc}
\usepackage[utf8]{inputenc}
\usepackage{amsmath,amssymb,amsfonts,amsthm,bm,mathtools,mathrsfs}
\usepackage{graphicx}
\usepackage{float}
\usepackage{booktabs}
\usepackage{enumitem}
\usepackage{microtype}
\usepackage{xcolor}
\definecolor{darkblue}{rgb}{0,0,.5}
\usepackage{url}
\usepackage[colorlinks=true,allcolors=darkblue]{hyperref}
\usepackage[capitalize,noabbrev]{cleveref}

\allowdisplaybreaks

\newtheorem{theorem}{Theorem}[section]
\newtheorem{lemma}[theorem]{Lemma}
\newtheorem{proposition}[theorem]{Proposition}
\newtheorem{corollary}[theorem]{Corollary}
\newtheorem{assumption}{Assumption}

\newcommand{\ind}{\mathbf{1}}
\newcommand{\var}{\mathsf{Var}}

\newcommand{\brc}[1]{\left\{{#1}\right\}}
\newcommand{\prn}[1]{\left({#1}\right)}

\newcommand{\norm}[1]{\left\|{#1}\right\|}
\newcommand{\abs}[1]{\left|{#1}\right|}
\newcommand{\rd}{\mathrm{d}}

\def\gA{{\mathcal{A}}}
\def\gE{{\mathcal{E}}}
\def\gF{{\mathcal{F}}}
\def\gM{{\mathcal{M}}}
\def\gP{{\mathcal{P}}}
\def\gS{{\mathcal{S}}}
\def\gT{{\mathcal{T}}}
\def\gZ{{\mathcal{Z}}}
\def\sP{{\mathscr{P}}}

\def\be{{\bm{e}}}
\def\bB{{\bm{B}}}
\def\bD{{\bm{D}}}
\def\bR{{\bm{R}}}
\def\bh{{\bm{h}}}

\def\bw{{\bm{w}}}
\def\bI{{\bm{I}}}
\def\bG{{\bm{G}}}

\def\bPi{{\bm{\Pi}}}
\def\bPhi{{\bm{\Phi}}}
\def\btheta{{\bm{\theta}}}
\def\bu{{\bm{u}}}
\def\bv{{\bm{v}}}
\def\bxi{{\bm{\xi}}}

\def\RB{{\mathbb R}}
\def\EB{{\mathbb E}}
\def\PB{{\mathbb P}}
\def\NB{{\mathbb N}}
\def\diag{{\operatorname{diag}}}
\newcommand{\cov}{\mathsf{Cov}}

%% file: 0_abstract.tex

Quantile temporal-difference learning (QTD) is an effective method for learning return distributions through quantile approximation, yet its finite-time behavior remains poorly understood. Its update is nonlinear and nonsmooth, and the stability needed for a sharp convergence rate holds only near the target. We establish a global high-probability last-iterate guarantee for synchronous tabular QTD under general positive, nonincreasing step-size sequences and arbitrary initialization in the natural parameter range. For polynomially decaying step sizes with exponent $a\in(0,1)$, the last iterate converges to the target at rate $T^{-a/2}$ in the infinity norm, up to logarithmic and lower-order terms. A suitably tuned harmonic schedule recovers the $T^{-1/2}$ statistical rate up to logarithmic factors. 
For the $m$-quantile representation, its $\infty$-Wasserstein error scales as $\sqrt{m/T}$ up to logarithmic factors, matching the leading polynomial dependence on the quantile resolution and sample size of the corresponding model-based estimator.
The proof uses a two-stage global-to-local argument. From arbitrary initialization, Bellman contraction and CDF monotonicity first bring the iterate close to the target, after which, a novel variance--drift matching argument sharpens the control of accumulated noise and local contraction reduces the remaining errors, yielding the sharp rate.
Simulations verify the predicted polynomial decay and assess the finite-time entrance bound.

%% file: 1_introduction.tex
\section{Introduction}

Distributional reinforcement learning (DRL) learns the distribution of future outcomes rather than only their expectation. 
By retaining the full distribution of possible returns, it provides a richer picture of the uncertainty 
induced by the environment and the policy.
Methods built on this idea have achieved strong performance in challenging games and control tasks, and have also been applied to real-world problems such as robotic manipulation and autonomous navigation
\citep{bellemare2017distributional,dabney2018distributional,dabney2018implicit,barthmaron2018distributed,bodnar2019quantile,bellemare2020autonomous}.

Many distributional methods build on temporal-difference (TD) learning, which uses short transitions to learn long-term returns without waiting for an episode to end \citep{sutton2018reinforcement}. Classical TD learning predicts the expected return, whereas distributional TD learning uses the same recursive principle to estimate its full distribution \citep{bellemare2017distributional}. 
One prominent approach is quantile TD learning (QTD), which represents the return distribution through a fixed collection of quantiles, with each quantile corresponding to a value below which a specified proportion of returns falls,
and learns their locations through quantile regression
\citep{dabney2018distributional,rowland2023analysis}. QTD thus reduces distributional learning to a collection of simple but interacting quantile updates.

Despite its empirical success, the finite-time behavior of QTD remains poorly understood. 
For nonparametric and categorical distributional TD learning, finite-sample analyses have established a minimax-optimal rate in the $1$-Wasserstein distance, showing that estimating a return distribution need not sacrifice statistical efficiency \citep{peng2024statistical}.
For QTD, almost-sure asymptotic convergence establishes that the quantile iterates eventually approach their target
\citep{rowland2023analysis}. These results, however, do not describe how QTD behaves over a finite run. This leaves a basic question: \emph{how quickly does QTD converge to its target?}

In this work, we develop global, high-probability finite-time convergence guarantees for the last-iterate estimator of synchronous tabular QTD under a generative model. 
Our results
cover a general class of positive, nonincreasing step-size sequences under explicit conditions, and the resulting bounds quantify both the transient effect of initialization and the stochastic error induced by the step-size schedule.
In particular, they yield guarantees for three common schedules: (i) a constant step size $\alpha_t=\eta$, (ii) a polynomially decaying step size $\alpha_t=c(t+1)^{-a}$ with $0<a<1$, and (iii) a harmonic step size $\alpha_t=c/(t+t_0)$. 
We show that, up to logarithmic factors, a constant step size forgets the initialization geometrically until reaching a noise floor of order $\sqrt{\eta}$ ; polynomial decay yields an error of order $T^{-a/2}$; and a suitably tuned harmonic schedule recovers the canonical $T^{-1/2}$ statistical rate.
Moreover, for an $m$-quantile representation, the harmonic schedule yields a $W_\infty$ error scaling as $\sqrt{m/T}$ up to logarithmic factors, matching the model-based estimator of \citet{cheng2026statistical} in its leading polynomial dependence on quantile resolution $m$ and sample size.

The main technical difficulty is that sharp control of the QTD recursion is available only near the target, while the iterate may start far away. A purely local analysis therefore cannot explain why the iterate reaches this region and remains there. We address this through a two-stage argument: a global analysis first drives the iterate into a suitable local neighborhood, after which a refined concentration argument controls its stochastic fluctuations and prevents exit with high probability.
Within the local regime, a key variance--drift matching property leads to a sharper rate than worst-case bounded-noise arguments. 
Intuitively, an extreme quantile receives a weaker pull toward its target, but its quantile score also fluctuates less. Our analysis keeps these two effects together. 

Our contributions are as follows.
\begin{enumerate}[leftmargin=*,label=(\roman*)]
\item We establish a global finite-sample, high-probability last-iterate
guarantee for synchronous tabular QTD.
A unified result for positive, nonincreasing step sizes
yields explicit guarantees for constant, polynomial, and harmonic schedules.
\item We develop a two-stage analysis that connects global convergence with
sharp local rate analysis: the first stage guarantees entrance into a local
region, while the second controls stochastic fluctuations and possible exit from that region.
\item We uncover a variance--drift matching property of the local QTD
dynamics. At an extreme quantile, a weak restoring drift is accompanied by
smaller quantile noise, yielding sharper stochastic fluctuation control than
worst-case bounded-noise analysis.
\end{enumerate}

\paragraph{Related work.}
Foundational work on DRL developed distributional Bellman operators and
categorical and quantile approximations
\citep{bellemare2017distributional,rowland2018analysis,dabney2018distributional}.
For quantile representations, \citet{rowland2023analysis} established asymptotic convergence of QTD, while \citet{cheng2026online} established asymptotic normality and studied statistical inference for QTD.
\citet{rowland2023statistical} studied the statistical benefits of quantile approximations for value estimation. Finite-sample analyses are more developed for categorical and nonparametric distributional methods \citep{peng2024statistical,peng2025finite}. QTD resembles nonsmooth quantile regression, but its evolving Bellman target prevents a direct application of fixed-objective nonsmooth SGD analyses \citep{shamir2013stochastic,harvey2019tight}. 
Supplementary background and related work are provided in Appendix~\ref{Appendix_related_work}.

\paragraph{Paper organization.}
Section~\ref{Section:preliminary} defines the model, the QTD target, and the
algorithm. Section~\ref{Section:analysis} states the general finite-sample
result and its consequences for the three step-size schedules.
Section~\ref{Section:proof_sketch} explains the two-stage proof, and
Section~\ref{Section:simulations} provides numerical verification. Supplemental background and related work, complete
proofs, and additional experiments are all deferred to the appendix.

%% file: 2_problem_setup.tex
\section{Problem Setup}
\label{Section:preliminary}

\subsection{Distributional policy evaluation}

Consider a discounted Markov decision process
$\gM=\langle\gS,\gA,\gP_R,P,\gamma\rangle$.  
Here $\gS$ is a finite state space, $\gA$ is a finite action space, $s,s'\in\gS$ denote states, and
$a\in\gA$ denotes an action.  Given $(s,a)$,
$\gP_R(\cdot\mid s,a)$ is the reward distribution, supported on $[0,1]$, and
$P(\cdot\mid s,a)$ is the distribution of the next state.  
$\gamma\in(0,1)$ is the discount factor.
Fix a policy $\pi$, where $\pi(a\mid s)$ is the probability of choosing action
$a$ in state $s$. At time $t$, the policy draws
$A_t\sim\pi(\cdot\mid S_t)$, the environment draws
$R_t\sim\gP_R(\cdot\mid S_t,A_t)$ and
$S_{t+1}\sim P(\cdot\mid S_t,A_t)$. Starting from $S_0=s$, the
return is $G^\pi(s)=\sum_{t=0}^\infty\gamma^tR_t$.  We denote its distribution
by $\eta^\pi(s)=\operatorname{Law}(G^\pi(s))$ and collect these statewise
distributions as $\bm\eta^\pi=(\eta^\pi(s)\colon s\in\gS)$.  Because rewards lie in
$[0,1]$, the return lies in $[0,(1-\gamma)^{-1}]$.

For a state-indexed collection of probability measures
$\bm\eta=(\eta(s)\colon s\in\gS)$, let $b_{r,\gamma}(x)=r+\gamma x$.  We use
$(b_{r,\gamma})_\#\nu$ for the distribution of $r+\gamma X$ when $X\sim\nu$.
The distributional Bellman operator is
\begin{equation}
(\gT^\pi\bm\eta)(s)
=\sum\nolimits_{a\in\gA,,s'\in\gS}
\pi(a\mid s)P(s'\mid s,a)
\int_{\mathbb R}(b_{r,\gamma})_\#\eta(s'),\gP_R(\rd r\mid s,a).
  \label{eq:distributional_bellman}
\end{equation}
Its unique fixed point is the collection of return distributions
$\bm\eta^\pi$. 

We next define the finite quantile representation used by QTD.  
For a probability measure $\nu$ on
$\RB$, let $F_\nu(x)=\nu((-\infty,x])$ be its cumulative distribution function
(CDF) and let $F_\nu^{-1}(\tau)=\inf\{x\colon F_\nu(x)\ge\tau\}$ be its quantile
function.  At the mid-quantile levels $\tau_i=(2i-1)/(2m)$, $i\in[m]\coloneq \{1,\ldots,m\}$, define
$\Pi_m\nu=m^{-1}\sum_{i=1}^m\delta_{F_\nu^{-1}(\tau_i)}$, where $\delta_x$
denotes a point mass at $x$.  For a state-indexed collection $\bm\eta$, let
$\bPi_m$ apply the projection statewise, so that
$(\bPi_m\bm\eta)(s)=\Pi_m\eta(s)$.

We next introduce the $W_\infty$ metric, under which the distributional Bellman operator is contractive and the quantile projection is nonexpansive.
For two probability measures $\mu$ and $\nu$ define
$W_\infty(\mu,\nu)=\sup_{u\in(0,1)}
\abs{F_\mu^{-1}(u)-F_\nu^{-1}(u)}$.  For state-indexed collections
$\bm\mu=(\mu(s)\colon s\in\gS)$ and $\bm\nu=(\nu(s)\colon s\in\gS)$, write
$\overline W_\infty(\bm\mu,\bm\nu)=
\max_{s\in\gS}W_\infty(\mu(s),\nu(s))$.
The metric $W_\infty$ is the largest horizontal separation between two
quantile functions and therefore measures the worst quantile-location error.
The 
additional measure-theoretic details are
collected in Appendix~\ref{Appendix_additional_preliminaries}.

\begin{proposition}[Bellman and projection stability]
\label{prop:bellman_projection_stability}
For all state-indexed collections of probability measures $\bm\eta$ and
$\bm\eta'$, the distributional Bellman operator and the quantile projection
satisfy
\begin{equation*}
  \overline W_\infty(\gT^\pi\bm\eta,\gT^\pi\bm\eta')
  \le\gamma\overline W_\infty(\bm\eta,\bm\eta'),\qquad
  \overline W_\infty(\bPi_m\bm\eta,\bPi_m\bm\eta')
  \le\overline W_\infty(\bm\eta,\bm\eta').
\end{equation*}
\end{proposition}
Proposition~\ref{prop:bellman_projection_stability} follows from
\citet[Proposition~4.15 and Lemma~5.25]{bdr2022}.  It implies that
$\bPi_m\gT^\pi$ is a $\gamma$-contraction and therefore has a unique fixed
point $\bm\eta_m$.  Thus, it identifies the target of QTD and lets us control
the distance between distributions before the iterates enter 
a neighborhood of $\eta_m$, where the conditional mean update admits a first-order Taylor expansion.
We write
$\eta_m(s)=m^{-1}\sum_{i=1}^m\delta_{\theta_m(s,i)}$, with
$\theta_m(s,1)\le\ldots\le\theta_m(s,m)$, and call the array
$\bm\theta_m=(\theta_m(s,i)\colon s\in\gS,i\in[m])$ the quantile-projected Bellman
fixed point.  More generally, an array
$\bm\theta=(\theta(s,i)\colon s\in\gS,i\in[m])\in\RB^{\gS\times[m]}$ represents the
statewise distributions
$\eta_{\bm\theta}(s)=\frac{1}{m}\sum_{i=1}^m\delta_{\theta(s,i)}$; its entries need not be ordered.
In particular, $\bm\eta_m=\bm\eta_{\bm\theta_m}$.
The approximation error $\overline W_\infty(\bm\eta_m, \bm\eta^\pi)$ vanishes as $m\to\infty$ under our assumptions
\citep[Theorem~3.2]{cheng2026statistical}.
And we measure the statistical error by
$\|\bm\theta-\bm\theta_m\|_\infty$, which bounds the corresponding $W_\infty$ error:
$\overline W_\infty(\bm\eta_{\bm\theta},\bm\eta_m)
\le\|\bm\theta-\bm\theta_m\|_\infty$.
In the following, all the results for $\|\bm\theta-\bm\theta_m\|_\infty$ hold naturally for $\overline W_\infty(\bm\eta_{\bm\theta},\bm\eta_m)$.


\subsection{Synchronous QTD}

At iteration $t+1$, for each state $s\in\gS$, the generative model draws
independently across states
\begin{equation*}
  a^{(t+1,s)}\sim\pi(\cdot\mid s),\qquad
  r^{(t+1,s)}\sim\gP_R(\cdot\mid s,a^{(t+1,s)}),\qquad
  s^{\prime(t+1,s)}\sim P(\cdot\mid s,a^{(t+1,s)}).
\end{equation*}
The superscript $(t+1,s)$ records the iteration and the state being updated.
For a parameter array $\bm\theta$ and a threshold $z\in\RB$, define the
empirical Bellman-target CDF
\begin{equation*}
  \widehat F_{s,\bm\theta}^{(t+1)}(z)
  =\frac{1}{m}
  {\sum\nolimits_{j=1}^m}
  \ind\!\left\{
  r^{(t+1,s)}+\gamma\theta
  \bigl(s^{\prime(t+1,s)},j\bigr)<z
  \right\}.
\end{equation*}
Here $\ind\{\cdot\}$ denotes the indicator of an event.
For each \((s,i)\), the synchronous QTD update is
\begin{equation*}
  \theta^{(t+1)}(s,i)
  =\theta^{(t)}(s,i)+\alpha_t
  \left\{\tau_i-\widehat F_{s,\bm\theta^{(t)}}^{(t+1)}
  \bigl(\theta^{(t)}(s,i)\bigr)\right\},
\end{equation*}
where $(\alpha_t)_{t\ge0}$ are
deterministic, positive and nonincreasing step sizes.   
If too few sampled targets lie below the current coordinate
$\theta^{(t)}(s,i)$, the update moves it upward; if too many lie below it, the
update moves it downward. Each coordinate is therefore pushed toward its
prescribed quantile level $\tau_i$.
\citet{rowland2023analysis} established almost-sure convergence
of tabular QTD to the fixed point $\bm\theta_m$; our goal is to quantify this
convergence at finite time.
Throughout the general analysis, $(\alpha_t)_{t\ge0}$ is deterministic,
positive, and nonincreasing, with $\bar\alpha=\alpha_0$.
Section~\ref{Section:analysis} states the resulting guarantees for constant,
polynomially decreasing, and harmonic schedules.

\subsection{Regularity assumptions}

We impose the
following conditions on the reward distributions.

\begin{assumption}[Lipschitz reward densities]
\label{assump:density}
For every $(s,a)\in\gS\times\gA$, the reward distribution admits a density
$p_{s,a}$ on $(0,1)$. For some constants $C_0\ge1$ and $L\ge0$, uniformly over
$(s,a)\in\gS\times\gA$, $0<p_{s,a}(x)\le C_0$ for all $x\in(0,1)$ and
$p_{s,a}$ is $L$-Lipschitz on $(0,1)$. At the endpoints, either:
\begin{enumerate}[label=(\roman*),leftmargin=*]
  \item there is a constant $c_0>0$ such that
  $p_{s,a}(x)\ge c_0$ for every $(s,a)\in\gS\times\gA$ and $x\in(0,1)$; or
  \item let $\widetilde p_{s,a}$ denote the zero extension of $p_{s,a}$ to
  $\RB$: it equals $p_{s,a}$ on $(0,1)$ and zero elsewhere.  The function
  $\widetilde p_{s,a}$ is $L$-Lipschitz on $\RB$.  Moreover, there is a constant
  $\kappa\in(0,1/2]$ such that every $p_{s,a}$ is nondecreasing on
  $(0,\kappa)$ and nonincreasing on $(1-\kappa,1)$.
\end{enumerate}
\end{assumption}

Assumption~\ref{assump:density} imposes regularity conditions on the reward densities.
The upper bound $C_0$ prevents too much probability from being concentrated in a very short interval, while the Lipschitz constant $L$
prevents the density from changing abruptly.  
We allow two behaviors at the
endpoints of the reward range.  In case~(i), the density stays uniformly
positive throughout $(0,1)$.  In case~(ii), it may decrease to zero at either
endpoint, as happens for some beta distributions,
but it must do so without oscillating. 
These conditions ensure that, near the target quantiles, location errors induce a non-negligible change in the corresponding CDF values, yielding a corrective QTD update.

\begin{assumption}[No boundary quantile]
\label{Assumption_no_boundary_quantile}
Let $P^\pi(s'\mid s)=\sum_{a\in\gA}\pi(a\mid s)P(s'\mid s,a)$ be the
transition probability from $s$ to $s'$ under policy $\pi$. 
For every integer $m\ge1$, every $s,s'\in\gS$ with $P^\pi(s'\mid s)>0$,
and every $i,j\in[m]$, we
assume that $\theta_m(s,i)-\gamma\theta_m(s',j)\notin\{0,1\}$.
\end{assumption}

To interpret Assumption~\ref{Assumption_no_boundary_quantile}, observe that a
reward $r$ aligns the next-state value $\gamma\theta_m(s',j)$ with the current
quantile $\theta_m(s,i)$ precisely when
$r=\theta_m(s,i)-\gamma\theta_m(s',j)$.  The assumption simply says that none
of these relevant reward values is exactly $0$ or $1$, the two endpoints at
which the density may change its behavior. This excludes degenerate boundary configurations that would otherwise obstruct the local analysis around the target quantiles. 

%% file: 3_main_results.tex
\section{Main Results}
\label{Section:analysis}

Our main result gives a high-probability bound for the last QTD iterate under
any positive, nonincreasing step-size sequence.  
The result separates a burn-in requirement for reaching the local region from the subsequent error bound, which consists of the contracted error inherited at entrance and the accumulated sampling fluctuations.
We first state this general result and then specialize it to three step size schedules.  
To state the precise bounds, we need one property of the MDP.  Define
\begin{equation}
c_\gM \coloneq
\inf_{\substack{m\in\NB,\ s\in\gS,\ i\in[m],\ 0<|u|\le(1-\gamma)^{-1}}}
\frac{\big|F_{(\gT^\pi\bm\eta_m)(s)}(\theta_m(s,i)+u)-\tau_i\big|}
     {\tau_i(1-\tau_i)|u|}.
\label{eq:cM_definition}
\end{equation}
Under Assumptions~\ref{assump:density} and
\ref{Assumption_no_boundary_quantile}, $c_\gM>0$ and depends only on the MDP $\gM$
\citep[Lemma~B.1]{cheng2026statistical}. 
It measures how sensitively the Bellman-target CDF changes when a candidate quantile moves away from its target.
A larger \(c_{\mathcal M}\) means a stronger corrective signal in the QTD update.

\begin{theorem}[General nonincreasing step sizes]
\label{thm:global_nonasymptotic}
Suppose Assumptions~\ref{assump:density} and
\ref{Assumption_no_boundary_quantile} hold,
$\bm\theta^{(0)}\in[0,(1-\gamma)^{-1}]^{\gS\times[m]}$.
Assume further that the initial, and hence largest, step size satisfies
$\alpha_0C_0\le1$.
Fix $\delta\in(0,1)$ and $T\ge8$, and set
$u_T=\lfloor T/4\rfloor$ and $v_T=\lfloor T/2\rfloor$.
Let $\{u_T,u_T+1,\ldots,v_T-1\}$ be the entrance block and $\{v_T, v_T+1,\ldots, T\}$ be the local analysis block.
Define the accumulated step size and squared step size on the entrance block,
and the accumulated step size after that block, by
\[
A_T^{\mathrm{ent}}
\coloneq\sum\nolimits_{t=u_T}^{v_T-1}\alpha_t,\qquad
V_T^{\mathrm{ent}}
\coloneq\sum\nolimits_{t=u_T}^{v_T-1}\alpha_t^2,\qquad
A_T^{\mathrm{loc}}
\coloneq\sum\nolimits_{t=v_T}^{T-1}\alpha_t.
\]
We also define
\begin{equation*}
  \ell_T(\delta)=\log\frac{8|\gS|mT^2}{\delta},
  \qquad
  b_T(\delta)
  =2\sqrt{\frac{\alpha_{u_T}\ell_T(\delta)}
  {c_\gM(1-\gamma)}}
  +\frac{\alpha_{u_T}\ell_T(\delta)}3.
\end{equation*}
If the following two burn-in conditions hold:
\begin{subequations}
\label{eq:if-condition}
\begin{align}
  c_{\mathrm g}A_T^{\mathrm{ent}}
  &\ge D_{\mathrm g}+\frac{\beta}{2}V_T^{\mathrm{ent}}
  +\sqrt{2V_T^{\mathrm{ent}}\log\frac2\delta},
  &&\text{(Entrance)}\label{eq:entrance_condition}\\
  b_T(\delta)
  &\le
  \min\left\{\frac{r_{\mathrm{out}}}{16},
  \frac{\mu}{64L_h}\right\},
  &&\text{(Capture)}\label{eq:capture_condition}
\end{align}
\end{subequations}
where $L_h\coloneq L(1+\gamma)^2/2$, $\mu\coloneq c_\gM(1-\gamma)/(4m)$ is the local contraction rate, and
$r_{\mathrm{out}}$ is the local-region radius. $r_{\mathrm{out}}$ together with $c_{\mathrm g},D_{\mathrm g},\beta$ is defined in
Appendix~\ref{sec:explicit_proof_constants}.
Then, with probability at least $1-\delta$, it holds that
\begin{equation}
  \norm{\bm\theta^{(T)}-\bm\theta_m}_\infty
  \le
  \frac{r_{\mathrm{out}}}{2}
  \exp\left\{
  -\frac{15c_\gM(1-\gamma)}{64m}A_T^{\mathrm{loc}}
  \right\}
  +\frac{31}{30}b_T(\delta).
  \label{eq:generic_hp_bound}
\end{equation}
\end{theorem}

The first burn-in condition \eqref{eq:entrance_condition} says that the accumulated
progress $A_T^{\mathrm{ent}}$ is sufficient to overcome the initial error $D_{\mathrm g}$ and noise (the $V_T^{\mathrm{ent}}$ terms) and reach a
neighborhood of the target. 
The second \eqref{eq:capture_condition} says that subsequent sampling
fluctuations $b_T(\delta)$ are small enough to keep the iterate there.
This division into blocks is only a proof device; QTD runs without a restart or change of step-size rule.

Stage~I in Section~\ref{Section:proof_sketch} establishes entrance. Conditional
on entrance, Stage~II contracts the error remaining at entrance, producing the exponential
term in \eqref{eq:generic_hp_bound}, and controls the subsequent accumulated noise,
producing $b_T(\delta)$. 
We next specialize the general result to the three step-size schedules.

\begin{corollary}[Constant and polynomial step sizes]
\label{cor:constant_polynomial}
Under the assumptions of Theorem~\ref{thm:global_nonasymptotic}, the following
bounds hold whenever its two burn-in conditions \eqref{eq:entrance_condition}, \eqref{eq:capture_condition} are satisfied.
\begin{enumerate}[label=(\roman*),leftmargin=*]
  \item If $\alpha_t=\eta$, then, with probability at least $1-\delta$,
  \begin{equation*}
    \norm{\bm\theta^{(T)}-\bm\theta_m}_\infty
    \le C\exp\left\{-\frac{c_\gM(1-\gamma)\eta T}{Cm}\right\}
    +C\sqrt{\frac{\eta\ell_T(\delta)}{c_\gM(1-\gamma)}}
    +C\eta\ell_T(\delta).
  \end{equation*}
  \item If $\alpha_t=c(t+1)^{-a}$ with $a\in(0,1)$ and $cC_0\le1$, then,
  with probability at least $1-\delta$,
  \begin{equation}
    \norm{\bm\theta^{(T)}-\bm\theta_m}_\infty
    \le C
    \exp\left\{-\frac{c_\gM(1-\gamma)cT^{1-a}}{Cm}\right\}
    +C\sqrt{\frac{c\ell_T(\delta)}
    {c_\gM(1-\gamma)T^a}}
    +\frac{Cc\ell_T(\delta)}{T^a}.
    \label{eq:polynomial_hp_bound}
  \end{equation}
\end{enumerate}
\end{corollary}

For a constant step size, QTD forgets its initialization geometrically, but
every iteration continues to inject noise at the same scale. The iterate
therefore eventually fluctuates in a neighborhood whose leading radius scales as
$\sqrt{\eta}$. 
A smaller \(\eta\) shrinks this neighborhood but slows entrance, while the burn-in conditions require \(T\eta\) to be sufficiently large and \(\eta\ell_T(\delta)\) sufficiently small.

For polynomial decay, the decreasing steps also reduce the stochastic fluctuations, giving the leading last-iterate rate $T^{-a/2}$ up to logarithmic
factors. 
For every fixed \(a\in(0,1)\), the burn-in conditions hold for all sufficiently large \(T\).
As in the asymptotic analysis of \citet{rowland2023analysis}, square summability of the step sizes is not required. 
Here we further obtain a finite-time rate: for every \(a\in(0,1)\), including \(a\le1/2\), the last-iterate error decays as \(T^{-a/2}\) up to logarithmic factors.

The endpoint \(t^{-1}\) decay of the harmonic schedule requires a separate dyadic-block argument.
We divide time into blocks
whose endpoints double, such as $[N,2N)$, $[2N,4N)$, and $[4N,8N)$; these are
called \emph{dyadic blocks}.
For the harmonic step size, the individual steps
become smaller over time, but their sum within each such block remains of
constant order.  Thus, every block reduces the error inherited from the
previous block by a fixed factor, while the sampling noise becomes smaller.

\begin{corollary}[Harmonic step size]
\label{cor:harmonic_stepsize}
Under the assumptions of
Theorem~\ref{thm:global_nonasymptotic},
let $\alpha_t=c/(t+t_0)$. There is a universal constant $C_{\mathrm{harm}}>0$
such that, if $c\ge C_{\mathrm{harm}}\mu^{-1}$ and $t_0\ge cC_0$, then there
is a finite deterministic burn-in
$T_{\mathrm{harm}}(\delta)$, defined in \eqref{eq:definition_T_harm}, for which the following holds.  For every
$T\ge T_{\mathrm{harm}}(\delta)$, with probability at least $1-\delta$,
\begin{equation}
  \norm{\bm\theta^{(T)}-\bm\theta_m}_\infty
  \le C\left\{
  \sqrt{\frac{c\log(C|\gS|mT/\delta)}
  {c_\gM(1-\gamma)T}}
  +\frac{c\log(C|\gS|mT/\delta)}{T}\right\}.
  \label{eq:harmonic_hp_bound}
\end{equation}
\end{corollary}

The condition $c\ge C_{\mathrm{harm}}\mu^{-1}$  ensures that each dyadic block makes enough progress to forget the error inherited from the preceding block to keep pace with the decreasing noise,
while the offset $t_0$ keeps the early steps stable. 
The harmonic schedule recovers
the canonical $T^{-1/2}$ statistical rate for the last iterate, up to
logarithmic factors and the lower-order $T^{-1}$ term. For comparison,
polynomial decay gives $T^{-a/2}$, while a constant step size reaches a noise
floor of order $\sqrt{\eta}$.
Choosing $c\asymp\mu^{-1}$ and treating all other problem parameters as fixed, the bound scales as $\sqrt{m/T}$ up to logarithmic factors in $m$ and $T$. This matches the leading polynomial dependence on $m$ and sample size of the model-based estimator \citep{cheng2026statistical}.



%% file: 4_proof_sketch.tex
\section{Proof Sketch}
\label{Section:proof_sketch}


The proof uses two stages because different information is available at
different distances from the target. 
Far from $\bm\theta_m$, a first order Taylor approximation need not be accurate, but CDF monotonicity still shows that the mean update points towards the target. Stage~I uses this global directional information to reach a neighborhood of $\bm\theta_m$. Once the iterate is close, Stage~II uses a local Taylor expansion, leading to a linearized update map that is entrywise nonnegative and contractive. This keeps the iterate local and gives the sharp stochastic rate.
The stages correspond to the two
conditions in Theorem~\ref{thm:global_nonasymptotic} respectively.

In the following, we first write QTD in vector form, which separates its mean update from the
sampling noise. At iteration $t$, the synchronous batch contains one sampled
$(a,r,s^\prime)$ triplet for every state, denoted by
$\gZ_t=\{(a^{(t,s)},r^{(t,s)},s^{\prime(t,s)}):s\in\gS\}$. Set
$\gF_t=\sigma(\bm\theta^{(0)},\gZ_1,\ldots,\gZ_t)$, and let $F_{s,a}$ be
the CDF of the reward distribution $\gP_R(\cdot\mid s,a)$. Thus, $\gF_t$
records the initialization and all samples observed through time $t$. The mean and
sample update fields are
\begin{equation}
\begin{aligned}
&h_{s,i}(\bm\theta)
=\tfrac1m \sum\nolimits_{a\in\gA,\,s'\in\gS}
  \sum\nolimits_{j=1}^m
  \pi(a\mid s)P(s'\mid s,a)
  F_{s,a}\!\bigl(\theta(s,i)-\gamma\theta(s',j)\bigr)-\tau_i,\\
&\widehat h_{s,i}^{(t+1)}(\bm\theta)
=\tfrac1m \sum\nolimits_{j=1}^m
  \ind\!\bigl\{r^{(t+1,s)}
  +\gamma\theta(s^{\prime(t+1,s)},j)<\theta(s,i)\bigr\}-\tau_i .
\end{aligned}
\label{eq:qtd_mean_field}
\end{equation}
With
$\bxi^{(t+1)}=\widehat\bh^{(t+1)}(\bm\theta^{(t)})-
\bh(\bm\theta^{(t)})$, the vector QTD recursion is
\begin{equation}
\label{eq:vector-update-rule}
	\bm\theta^{(t+1)}
	=\bm\theta^{(t)}-\alpha_t
	\{\bh(\bm\theta^{(t)})+\bxi^{(t+1)}\},
	\qquad \EB[\bxi^{(t+1)}\mid\gF_t]=\bm0.
\end{equation}
Here $\bh(\bm\theta)$ is the average update direction and $\bh(\btheta_m)=0$, while
$\bxi^{(t+1)}$ is the centered error of the direction computed from the current
sample batch. Define $\be^{(t)}=\bm\theta^{(t)}-\bm\theta_m$, then
$\bh(\bm\theta^{(t)})=\bh(\bm\theta_m+\be^{(t)})$ and the recursion can be
studied directly in terms of the error vector $\be^{(t)}$.

\subsection{Stage I: the global error decreases until the local region is reached}

Outside the local region, Bellman structure still provides a global directional signal, though being non-smooth. In particular, if a signed coordinate $\varsigma e_{s,i}=\|\be\|_\infty$ for some $\varsigma\in\{-1,1\}$, then
\begin{equation*}
	\varsigma\Bigl(\{\theta(s,i)-\gamma\theta(s',j)\}
	-\{\theta_m(s,i)-\gamma\theta_m(s',j)\}\Bigr)
	=\varsigma\{e_{s,i}-\gamma e_{s',j}\}
	\ge (1-\gamma)\|\be\|_\infty.
\end{equation*}
CDF monotonicity converts this gap into a drift towards the target $\btheta_m$. The maximizing coordinate, however,
may change from one iteration to the next, so tracking any single coordinate
does not yield a stable potential.  We therefore aggregate all signed
coordinate errors through the smooth maximum
\begin{equation*}
    \Phi_\beta(\be)
    =\beta^{-1}\log\sum\nolimits_{s\in\gS,\ i\in[m]}
    \left(e^{\beta e_{s,i}}+e^{-\beta e_{s,i}}\right).
\end{equation*}
For large $\beta$, $\Phi_\beta$ closely approximates
$\norm{\be}_\infty$ while remaining smooth for a one-step Taylor
expansion. Lemma~\ref{lem:global_smooth_max} makes the preceding inward-drift
intuition quantitative.


\begin{lemma}[Global drift of the smooth maximum]
	\label{lem:global_smooth_max}
    There exist some positive constants $r_{\mathrm{out}}$, $c_\mathrm{g}$ and $\beta$ (specified in Appendix~\ref{sec:explicit_proof_constants}).
	For every $\be,\bu\in\RB^{\gS\times[m]}$,
	\begin{equation*}
    \label{eq:Phi}
		\norm{\be}_\infty
		\le\Phi_\beta(\be)
		\le\norm{\be}_\infty+\frac{\log(2|\gS|m)}{\beta},\quad
		\norm{\nabla\Phi_\beta(\be)}_1\le1,\quad
		\bu^\top\nabla^2\Phi_\beta(\be)\bu
		\le\beta\norm{\bu}_\infty^2.
	\end{equation*}
	Moreover, if $\be=\bm\theta-\bm\theta_m$ and
	$\norm{\be}_\infty\ge r_{\mathrm{out}}/2$, then
	$\left\langle\nabla\Phi_\beta(\be),\bh(\bm\theta)\right\rangle
	\ge c_{\mathrm g}>0.$
\end{lemma}
In Lemma \ref{lem:global_smooth_max}, the first three bounds make $\Phi_\beta$ a smooth proxy for the maximum error,
while the final inequality $\left\langle\nabla\Phi_\beta(\be),\bh(\bm\theta)\right\rangle\ge c_{\mathrm g}$ ensures a uniform first-order decrease outside the
local ball.
As the expected QTD update is $-\alpha\bh(\bm\theta)$, a positive lower bound on $\langle\nabla\Phi_\beta(\be),\bh(\bm\theta)\rangle$ ensures that this update decreases $\Phi_\beta$ to first order.
We then show that this decrease dominates the accumulated
sampling noise. Fix $0\le u<v$ and define
\begin{equation*}
	A_{u,v}=\textstyle\sum_{t=u}^{v-1}\alpha_t,
	\qquad
	V_{u,v}=\textstyle\sum_{t=u}^{v-1}\alpha_t^2,
	\qquad
	\tau_{u,v}=\inf\bigl\{u\le t\le v:
	\|\be^{(t)}\|_\infty\le r_{\mathrm{out}}/2\bigr\}.
\end{equation*}
with $\inf\varnothing=\infty$. Here, $A_{u,v}$ measures the deterministic
progress available during the step block $[u, v)$, whereas $V_{u,v}$ is the variance scale of its accumulated sampling noise.
For $u\le t<v$ with $t<\tau_{u,v}$, set
$Z_{t+1}=-\langle\nabla\Phi_\beta(\be^{(t)}),\bxi^{(t+1)}\rangle$.
Then $\EB[Z_{t+1}\mid\gF_t]=0$, $|Z_{t+1}|\le1$, and Taylor's expansion with
Lemma~\ref{lem:global_smooth_max} gives
\begin{equation*}
\Phi_\beta(\be^{(t+1)})
\le \Phi_\beta(\be^{(t)})-c_{\mathrm g}\alpha_t
+\alpha_tZ_{t+1}+\beta\alpha_t^2/2.
\end{equation*}
This is the desired global decay: the first negative term is \textit{the inward progress}, while the last two terms are the sampling fluctuation and the
second-order cost. If no entrance occurs by time $v$, 
we have $\Phi_\beta(\be^{(v)}) \ge \norm{\be^{(v)}}_\infty > r_{\mathrm{out}}/2$,
summing over the block
forces $\sum_{t=u}^{v-1}\alpha_tZ_{t+1}
\ge c_{\mathrm g}A_{u,v}
-\bigl(\Phi_\beta(\be^{(u)})-r_{\mathrm{out}}/2\bigr)
-\frac\beta2V_{u,v}.$
Azuma--Hoeffding thus gives the following entrance bound.

\begin{lemma}[Late entrance]
	\label{lem:finite_time_entrance}
	For any pair $0\le u<v$,
    	\begin{equation*}
		\label{eq:tau-bound}
		\PB(\tau_{u,v}=\infty)
		\le
		\exp\bigl\{
-\bigl[c_{\mathrm g}A_{u,v}-D_{\mathrm g}
-(\beta/2)V_{u,v}\bigr]_+^2/(2V_{u,v})
\bigr\},
	\end{equation*}
    where $[x]_+:=\max\{x,0\}$ denotes the positive part of $x$, $D_\mathrm{g}>0$ is defined in Appendix~\ref{sec:explicit_proof_constants}.
\end{lemma}

This bound makes the role of Stage~I explicit. 
The term
$c_{\mathrm g}A_{u,v}$ is the inward progress supplied by the global drift,
$D_{\mathrm g}$ bounds the initial distance in the smooth potential, and
$\beta V_{u,v}/2$ is the second-order cost of using that potential. 
Once we have $c_{\mathrm g}A_{u,v}
\ge D_{\mathrm g}+\frac\beta2V_{u,v}
+\sqrt{2V_{u,v}\log(1/\delta)}$, the
iterate enters the local ball by time $v$ with probability at least
$1-\delta$.
This is the moving-block point emphasized after
Corollary~\ref{cor:constant_polynomial}: no global summability condition on
squared steps is needed.

\subsection{Stage II: local errors contract and recent noise determines the rate}

\paragraph{Local linear structure.}
After entrance, a fixed drift towards the target is no longer enough to describe the rate: we must determine how quickly an earlier error is forgotten and how the noise from different iterations accumulates. 
This is where a local linearization of the mean field becomes useful.
Write $d_{s,i}=p_{(\gT^\pi\bm\eta_m)(s)}(\theta_m(s,i))$ for the Bellman-target density at a target quantile and set $\bD_m=\diag_{s,i}(d_{s,i})$.
We also introduce a nonnegative matrix $\bB_m$, which records how successor-quantile perturbations
change the current update and is given in
Appendix~\ref{sec:local_matrix_entries} due to space limits. The Jacobian of the mean field at \(\bm\theta_m\) then becomes
$\bG_m=\bD_m-\gamma\bB_m$.

\begin{lemma}[Local linear structure]
	\label{lem:jacobian_summarize}
	$\bh$ is continuously differentiable on the
	$r_{\mathrm{out}}$-neighborhood of $\bm\theta_m$, with
	$\nabla\bh(\bm\theta_m)=\bG_m$ and $\bB_m\bm1=\bD_m\bm1$.
	If $0\le\alpha C_0\le1$, where $\alpha$ denotes any step size, then
	$\bI-\alpha\bG_m\ge\bm0$ and
	$\norm{\bI-\alpha\bG_m}_\infty\le1-\alpha\mu$, where
	$\mu\coloneq c_\gM(1-\gamma)/(4m)$ is the local contraction rate.
	Moreover, whenever $\norm{\be}_\infty\le r_{\mathrm{out}}$,
	$\bh(\bm\theta_m+\be)=\bG_m\be+\bR(\be)$, where
	$\norm{\bR(\be)}_\infty\le
	L(1+\gamma)^2\norm{\be}_\infty^2/2$.
\end{lemma}

Lemma \ref{lem:jacobian_summarize} shows how the Jacobian $\bG_m$ controls the local dynamics.
The entrywise inequality $\bI-\alpha\bG_m\ge\bm0$ makes the propagation weights nonnegative, and the norm bound makes their total influence contract.
The quadratic remainder $\norm{\bR(\be)}_\infty$ is the small difference between the exact nonlinear
update and this linear approximation.

To track the accumulated contraction effect, we define the linear propagation matrices $\bPhi_{r,k}$ by
\begin{equation*}
  \bPhi_{r,k}
  =\textstyle\prod_{j=r}^{k-1}(\bI-\alpha_j\bG_m),\qquad
  \bPhi_{r,k}\ge\bm0,\qquad
  \norm{\bPhi_{r,k}}_\infty
  \le e^{-\mu\sum_{j=r}^{k-1}\alpha_j}.
\end{equation*}
The two bounds serve different purposes. Nonnegativity allows each row to be used as nonnegative weights below, while contraction makes the influence of older errors decay over time. The matrix $\bPhi_{r,k}$ therefore tracks how a perturbation introduced at time $r$ contributes to the error at time $k$.

\paragraph{Local error recursion.}
Let $n$ be an entrance time with
$\norm{\be^{(n)}}_\infty\le r_{\mathrm{out}}/2$, and let $\sigma$ be the first
subsequent exit from the $r_{\mathrm{out}}$-ball. Before $\sigma$, the update rule \eqref{eq:vector-update-rule} as well as Lemma \ref{lem:jacobian_summarize} gives
\begin{equation}
  \be^{(t+1)}=(\bI-\alpha_t\bG_m)\be^{(t)}
  -\alpha_t\bxi^{(t+1)}-\alpha_t\bR(\be^{(t)}).
  \label{eq:qtd_decomposition}
\end{equation}
For $n<k\le\sigma$, unrolling this recursion gives
\begin{equation}
\label{eq:error-recursion}
\be^{(k)}
=
\underbrace{
  \bPhi_{n,k}\be^{(n)}
}_{\text{entrance error}}
-\underbrace{
  {\textstyle\sum_{t=n}^{k-1}}
  \alpha_t\bPhi_{t+1,k}\bxi^{(t+1)}
}_{\text{sampling noise}}
-\underbrace{
  {\textstyle\sum_{t=n}^{k-1}}
  \alpha_t\bPhi_{t+1,k}\bR(\be^{(t)})
}_{\text{nonlinear remainder}}.
\end{equation}
Thus, \eqref{eq:error-recursion} summarizes Stage~II in one line: the error
present at entrance is damped (due to the bound for $\norm{\bPhi_{r,k}}_\infty$), new sampling noise is propagated through the
same dynamics, and curvature contributes a quadratic correction.
Writing $L_h=L(1+\gamma)^2/2$, contraction bounds the first term by
$e^{-\mu\sum_{j=n}^{k-1}\alpha_j}\norm{\be^{(n)}}_\infty$. The third is an
exponentially weighted sum of quadratic local errors because
$\norm{\bR(\be)}_\infty\le L_h\norm{\be}_\infty^2$, so it can be absorbed
locally. The main difficulty is the sampling-noise term.

\paragraph{Controlling the sampling-noise term.}
Fix a coordinate $(s,i)$ and times $n\le r<k\le T$. For $r\le t<k$, define
$\bv_t^\top=\bm e_{(s,i)}^\top\bPhi_{t+1,k}$ and
$u_t=\bv_t^\top\bm1$. Here, $\bm e_{(s,i)}$ selects the output coordinate,
$\bv_t$ describes how sampling noise at iteration $t$ propagates to coordinate $(s,i)$ at time $k$, and $u_t$ is its total weight.
Intuitively, bounding $\EB[(\bv_t^\top\bxi^{(t+1)})^2\mid\gF_t]$ through $u_t$ directly would lead to a loose analysis. 
Instead, we relate it to the actual one-step reduction $u_t-u_{t-1}$, whose sum over $t$ telescopes. The
next lemma introduces the variance--drift matching property, and establishes this relation by showing that the variance and weight reduction share the same density-weighted factor.




\begin{lemma}[Variance--drift matching]
\label{lem:variance_drift_matching}
Suppose $\norm{\bm\theta^{(t)}-\bm\theta_m}_\infty\le r_{\mathrm{out}}$.
For any nonnegative $\bv \ge 0$,
\begin{subequations}
\begin{align}
\bv^\top\bm1-\bv^\top(\bI-\alpha_t\bG_m)\bm1
&=(1-\gamma)\alpha_t\bv^\top\bD_m\bm1,
\label{eq:drift_matching}\\
\EB[(\bv^\top\bxi^{(t+1)})^2\mid\gF_t]
&\le 2c_\gM^{-1}(\bv^\top\bm1)(\bv^\top\bD_m\bm1).
\label{eq:variance_matching}
\end{align}
\end{subequations}
\end{lemma}

The common factor $\bv^\top\bD_m\bm1$ is the key. It determines both the one-step loss of propagation weight in \eqref{eq:drift_matching} and the conditional noise variance in \eqref{eq:variance_matching}.
Consequently, a direction $\bv_t$ with smaller restoring drift (i.e., $u_t-u_{t-1}$) would contribute less conditional stochastic variation in $\bv_t^\top\bxi^{(t+1)}$ given $\mathcal{F}_t$.
Indeed, applying the lemma with $\bv=\bv_t$ and using step-size monotonicity gives
\begin{equation*}
\sum\nolimits_{t=r}^{k-1}\alpha_t^2
\EB[(\bv_t^\top\bxi^{(t+1)})^2\mid\gF_t]
\le
\bigl(2\alpha_r/[c_\gM(1-\gamma)]\bigr)
\sum\nolimits_{t=r}^{k-1}u_t(u_t-u_{t-1})
\le
2\alpha_r/[c_\gM(1-\gamma)],
\end{equation*}
where the last inequality follows from $0\le u_t\le1$ and the telescoping of
$u_t-u_{t-1}$. Thus, the accumulated variance is controlled by the actual
loss of propagation weight.


The argument above controls one coordinate $(s, i)$ and one time interval $[r, k)$.
Lemma~\ref{lem:uniform_linearized_martingale} in the appendix and a union bound
give, conditionally on $\gF_n$ with probability at least $1-\delta$,
\begin{equation}
\label{eq:highprob-event}
  \max\nolimits_{n\le r<k\le T}
  \bigl\|\sum\nolimits_{t=r}^{k-1}\alpha_t\bPhi_{t+1,k}
  \ind\{t<\sigma\}\bxi^{(t+1)}\bigr\|_\infty
  \le b_{n,T}(\delta).
\end{equation}
The appendix gives the explicit logarithmic factors in $b_{n,T}(\delta)$, whose leading scale is $\sqrt{\alpha_n/[c_\gM(1-\gamma)]}$, with a lower-order term linear in $\alpha_n$.
The indicator halts the sum at a possible exit, ensuring that the local assumptions hold. For $k\le\sigma$, it equals one throughout the sum, which is then exactly the sampling-noise term in \eqref{eq:error-recursion}.



\paragraph{Staying inside the local ball.}


We now show that the uniform noise bound prevents the iterate from leaving
the local region. Stage~I reaches the ball of radius
$r_{\mathrm{out}}/2$, while the local expansion remains valid up to radius
$r_{\mathrm{out}}$. On the event \eqref{eq:highprob-event}, if the first exit
occurred at some $\sigma\le T$, then \eqref{eq:error-recursion} with
$k=\sigma$, together with contraction and the quadratic remainder bound in Lemma \ref{lem:jacobian_summarize},
would imply $\|\be^{(\sigma)}\|_\infty<r_{\mathrm{out}}$, contradicting the definition of
$\sigma$. The next lemma formalizes this argument.

\begin{lemma}[Local capture]
\label{lem:variance_sensitive_local}
Let $n\le T$ be an $(\gF_t)$-stopping time satisfying
$\norm{\bm\theta^{(n)}-\bm\theta_m}_\infty\le r_{\mathrm{out}}/2$.
If $b_{n,T}(\delta)\le\min\{r_{\mathrm{out}}/16,\mu/(64L_h)\}$, where
$L_h=L(1+\gamma)^2/2$ and
$\mu/L_h=+\infty$ when $L_h=0$, then,
conditionally on $\gF_n$, with probability at least
$1-\delta$, the iterates remain in the $r_{\mathrm{out}}$-neighborhood
through time $T$, and
\begin{equation}
\label{eq:local-iterate}
\|\bm\theta^{(T)}-\bm\theta_m\|_\infty
\le \|\bm\theta^{(n)}-\bm\theta_m\|_\infty
e^{-(15/16)\mu\sum_{t=n}^{T-1}\alpha_t}
+b_{n,T}(\delta)+(32L_h/(15\mu))b_{n,T}^2(\delta).
\end{equation}
\end{lemma}



This completes Stage~II. The three terms on the right-hand side of \eqref{eq:local-iterate} are the contracted entrance error, the accumulated sampling noise $b_{n,T}(\delta)$, and the quadratic nonlinear correction of order $b_{n,T}^2(\delta)$, respectively. Once
$b_{n,T}(\delta)$ is sufficiently small, the nonlinear correction is lower order, and the first-exit argument ensures that the iterates remain in the local region.

\subsection{Completing the proof of Theorem~\ref{thm:global_nonasymptotic}}

Now, we are ready for the final proof.
Set $u=\lfloor T/4\rfloor$ and $v=\lfloor T/2\rfloor$, where the first quarter is a warm-up, the next quarter is the entrance window, and the final half is the local window. 
The entrance condition~\eqref{eq:entrance_condition} gives an entrance time
$n\in[u,v]$ with failure probability at most $\delta/2$.
Step-size monotonicity gives
$b_{n,T}(\delta/2)\le b_T(\delta)$, so the capture
condition~\eqref{eq:capture_condition} allows us to apply
Lemma~\ref{lem:variance_sensitive_local} from time $n$ onward.
Step-size monotonicity further bounds the accumulated local step size from below by
$\sum_{t=n}^{T-1}\alpha_t \ge A_T^{\mathrm{loc}}$, and the quadratic remainder is absorbed into the
noise term.
 Consequently, for some constant $C>0$, we have $\norm{\be^{(T)}}_\infty\le
C\exp\{-C^{-1}\mu A_T^{\mathrm{loc}}\}+Cb_T(\delta)$.
Finally, a union bound yields \eqref{eq:generic_hp_bound} with probability
at least $1-\delta$, completing the proof.
 The explicit constants and remaining details, as well as the proofs of the corollaries, are given in Appendix~\ref{sec:proof_stepsize_specializations}.

%% file: 5_simulations.tex
\section{Numerical Verification}
\label{Section:simulations}

\begin{figure}[t!]
  \centering
  \includegraphics[width=\textwidth]{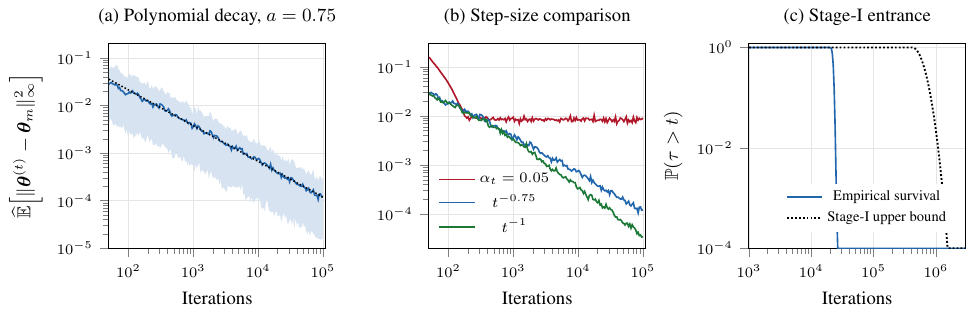}
  \caption{Numerical verification of the finite-time theory. (a) Mean squared
  sup-norm error for $a=0.75$, with the 10th--90th percentile band and a
  fixed-slope reference. (b) Comparison of the three step-size classes.
  (c) Empirical Stage-I non-entrance probability and the bound from
  Lemma~\ref{lem:finite_time_entrance}.}
  \label{fig:qtd_simulation_summary}
\end{figure}

Figure~\ref{fig:qtd_simulation_summary} provides three numerical checks of our
finite-time theory.  Panels~(a)--(b) use a one-state, one-action MDP with
uniform rewards, $\gamma=0.5$, $m=7$, and 200 independent trajectories up to
$T=10^5$.  For $\alpha_t=4/(t+20)^{0.75}$, Panel~(a) compares the mean squared
error with a reference line whose slope is fixed at the predicted value
$-0.75$; the two curves are nearly parallel.  Panel~(b) compares constant,
polynomial, and harmonic step sizes.  The constant schedule reaches a noise
floor, whereas the decreasing schedules continue to improve, with the
harmonic schedule performing best at long horizons, consistently with its
$T^{-1/2}$ raw-error rate.
Panel~(c) checks the Stage-I entrance bound in
Lemma~\ref{lem:finite_time_entrance}.  
The empirical non-entrance probability
stays below the initialization-specific upper bound; the gap reflects the conservativeness of the theoretical bounds.
Appendix~\ref{sec:additional_simulations} gives the full protocol, additional
polynomial exponents, and reproducibility details.


%% file: 6_conclusion.tex
\section{Conclusion}

We have established a global high-probability last-iterate bound for synchronous
tabular QTD under general nonincreasing step sizes. Constant steps forget the
initialization geometrically before reaching a noise floor, polynomial decay
gives the rate $T^{-a/2}$ up to logarithmic factors, and a tuned harmonic
schedule recovers the canonical $T^{-1/2}$ rate.
In particular, the tuned harmonic schedule yields a $W_\infty$ error scaling as $\sqrt{m/T}$ up to logarithmic factors in $m$ and $T$, matching the model-based statistical rate. 
The proof cleanly separates global and local roles: order
monotonicity and Bellman contraction localize the iterates, after which a
positive, contractive linearization and variance--drift matching yield the
sharp stochastic fluctuation.

%% file: tex/S_supplementary_related_work.tex
\section{Supplementary Background and Related Work}
\label{Appendix_related_work}

\paragraph{Nonsmooth SGD and QTD.}
For a fixed scalar response $Y$, a $\tau$-quantile minimizes the expected
pinball loss
\begin{equation*}
  q\longmapsto \EB\bigl[\rho_\tau(Y-q)\bigr],
  \qquad
  \rho_\tau(u)=u\bigl(\tau-\ind\{u<0\}\bigr).
\end{equation*}
A sample subgradient with respect to $q$ is $\ind\{Y<q\}-\tau$; under our
continuous-density assumptions, averaging gives the CDF mismatch
$F_Y(q)-\tau$. QTD steps in the negative of this direction. Nonsmooth SGD
theory provides averaged- and last-iterate
guarantees for fixed convex objectives under several regimes
\citep{shamir2013stochastic,harvey2019tight}.

The important difference is that QTD does not repeatedly sample from one
fixed response distribution. To make this distinction explicit, freeze a
quantile array $\bm\vartheta$ and, for state $s$, let
\begin{equation*}
  Y_{\bm\vartheta}(s)=R+\gamma\vartheta(S',J),
  \qquad J\sim\operatorname{Unif}([m]),
\end{equation*}
where the action, reward, and next state are sampled according to the policy
and MDP, and $J$ is independent of this transition. For fixed
$\bm\vartheta$, averaging over $j$ is the conditional expectation over $J$,
so the mean increment at coordinate $(s,i)$ is the negative of a subgradient
in $q$ of
\begin{equation*}
  L_{s,i}(q;\bm\vartheta)
  =\EB\bigl[\rho_{\tau_i}(Y_{\bm\vartheta}(s)-q)\bigr].
\end{equation*}
The cited nonsmooth-SGD analyses repeatedly use noisy subgradients of one fixed
convex objective. QTD instead takes this negative-subgradient step at
$q=\theta(s,i)$ while simultaneously setting $\bm\vartheta=\bm\theta$ and
updating the full array. It varies the prediction coordinate while treating
the target as fixed, even though the target distribution was generated by the
current array. This semi-gradient recursion is generally not SGD on one fixed
convex objective, so fixed-objective bounds do not directly control distance
to the coupled quantile-projected Bellman fixed point. Our global localization
and local stability arguments handle this feedback.

\paragraph{Classical quantile estimation and regression.}
Quantiles admit two complementary descriptions. Minimizers of expected
pinball loss satisfy $F(q^-)\le\tau\le F(q)$; at a regular quantile with a
continuous, strictly increasing CDF, this reduces to $F(q)-\tau=0$. The loss
view underlies regression quantiles and modern quantile regression
\citep{koenker1978regression,koenker2005quantile}. The estimating-equation view
underlies the local expansions used in quantile asymptotics. Bahadur and Kiefer
representations quantify how empirical CDF fluctuations translate into
sample-quantile errors \citep{bahadur1966note,kiefer1967bahadur}. Knight's
identity \citep{knight1998limiting} is a standard tool for expanding the
nonsmooth quantile-regression objective: it separates a pinball-loss difference
into a linear score and an integral remainder.

The same conversion is visible in our local analysis. A positive density near
a target quantile turns a small location error into a proportional CDF error,
which supplies the restoring drift of the QTD recursion. Unlike ordinary
quantile regression, however, that CDF is a Bellman-target CDF assembled from
all successor-state quantiles. The local derivative is therefore a coupled
matrix rather than a collection of independent scalar slopes.

\paragraph{Recursive quantile estimation and stochastic approximation.}
The classical Robbins--Monro procedure estimates the solution of a mean
equation from noisy observations and provides the basic stochastic
approximation template for online quantile estimation
\citep{robbins1951stochastic}. Recursive procedures tailored to quantiles
include the space-efficient estimator of \citet{tierney1983space}; more recent
work gives nonasymptotic confidence bounds for recursive quantile estimates
\citep{chen2023recursive}. General stochastic approximation theory studies
stability, limiting dynamics, and averaged iterates under broad conditions
\citep{kushner2003stochastic,benaim2006dynamics,polyak1992acceleration}.
These results provide scalar and asymptotic analogues for QTD. Ordinary
recursive quantile estimation samples repeatedly from a fixed response
distribution, whereas QTD changes its Bellman-target distribution with the
current array. We additionally need global finite-time control of this coupled
recursion and a high-probability last-iterate rate. Our two stages separate
these tasks: global drift establishes entrance, and variance-sensitive local
analysis gives the rate afterward.

Most directly, \citet{rowland2023analysis} established almost-sure asymptotic
convergence of tabular QTD, while \citet{cheng2026online} developed functional
limit theorems and online inference for QTD. Our question is complementary: we
seek a global high-probability finite-time bound for the last iterate from
arbitrary initialization and identify the variance-sensitive mechanism that
governs its local fluctuations.

\paragraph{Quantile representations in distributional reinforcement learning.}
Quantile regression distributional RL introduced a fixed grid of learned
quantile locations as an alternative to categorical value distributions
\citep{dabney2018distributional}. Subsequent work developed implicit quantile
networks, fully parameterized quantile functions, and non-crossing
architectures for more flexible representations
\citep{dabney2018implicit,yang2019fully,zhou2020noncrossing}. The broader link
between distributional statistics, finite collections of learned statistics,
and imputation strategies, which reconstruct a distribution from those
statistics, was studied by \citet{rowland2019statistics}. Much of this line
emphasizes richer representations and function-approximation algorithms. Our
fixed-grid tabular setting lets us isolate how Bellman coupling, global
localization, and variance--drift matching interact in finite time.
Complementary finite-sample analyses cover nonparametric and categorical
distributional TD, including linear function approximation
\citep{peng2024statistical,peng2025finite}; these representations have update
structures different from the coupled quantile-score recursion studied here.

For the same quantile-projected target, \citet[Theorems~3.1--3.2]{cheng2026statistical}
give a model-based $\widetilde O(\sqrt{m/n})$ $\overline W_\infty$
estimation bound and a modulus-of-continuity approximation bound to the
exact return law. Their estimator forms an empirical model from $n$
independent samples per state--action pair, or $n|\gS||\gA|$ samples in
total. Our synchronous QTD uses one fresh transition per state per iteration,
or $T|\gS|$ samples over $T$ iterations. With gain parameter
$c\asymp\mu^{-1}$, under the burn-in condition, Corollary~\ref{cor:harmonic_stepsize} matches their leading
polynomial dependence on quantile resolution and the respective sample
counts in the statistical rate. 

More explicitly, their approximation bound is
\begin{equation*}
 \overline W_\infty(\bm\eta_m,\bm\eta^\pi)
 \le\frac1{1-\gamma}\max_{s\in\gS}\omega_s\!\left(\frac1{2m}\right),
 \qquad
 \omega_s(\rho)=\sup_{\substack{u,v\in(0,1)\\|u-v|\le\rho}}
 \left|F_{\eta^\pi(s)}^{-1}(u)-F_{\eta^\pi(s)}^{-1}(v)\right|.
\end{equation*}
Under the reward assumptions here, the return quantile functions extend
continuously to the endpoints and hence are uniformly continuous. Therefore
$\max_s\omega_s(1/(2m))\to0$. Combining this existing approximation bound
with our statistical rates bounds the total
$\overline W_\infty$ error to $\bm\eta^\pi$.

%% file: tex/C_additional_preliminaries.tex
This appendix records background definitions that are useful for locating the
paper within distributional reinforcement learning but are not needed to
state the main theorem.

Let $\sP$ be the space of probability measures on $\RB$.  For
$p\in[1,\infty)$ and $\mu,\nu\in\sP$ with finite $p$th moments, their
$p$-Wasserstein distance is
\begin{equation*}
  W_p(\mu,\nu)
  =\left(\int_0^1
  \abs{F_\mu^{-1}(u)-F_\nu^{-1}(u)}^p\,\rd u\right)^{1/p}.
\end{equation*}
For any (possibly extended) metric $d$ on $\sP$, its statewise supremum
extension is
\begin{equation*}
  \overline d(\bm\mu,\bm\nu)=\sup_{s\in\gS}d(\mu(s),\nu(s)).
\end{equation*}
The main text uses the case $d=W_\infty$ because both the distributional
Bellman operator and the midpoint quantile projection are stable in that
metric.

For a measurable map $g:\RB\to\RB$, the pushforward $g_\#\mu$ is defined by
$(g_\#\mu)(A)=\mu(g^{-1}(A))$ for every Borel set $A$.  Therefore, the
reward-mixture term in the distributional Bellman operator is interpreted as
the measure satisfying
\begin{equation*}
  \left[\int (b_{r,\gamma})_\#\eta(s')\,
  \gP_R(\rd r\mid s,a)\right](A)
  =\int (b_{r,\gamma})_\#\eta(s')(A)\,
  \gP_R(\rd r\mid s,a).
\end{equation*}
This identity makes explicit that the integral in
Equation~\eqref{eq:distributional_bellman} is a mixture of probability
measures rather than an ordinary scalar integral.

%% file: tex/A_omitted_proofs.tex
\subsection{Explicit constants for the main results}
\label{sec:explicit_proof_constants}

This subsection records the proof-level constants suppressed from
Sections~\ref{Section:analysis} and~\ref{Section:proof_sketch}.  First define the distance between the
fixed-point density arguments and the endpoints of the reward support by
\begin{equation*}
  \Delta_m=\min_{\substack{P^\pi(s'\mid s)>0\\ i,j\in[m]}}
  \min\left\{
  \abs{\theta_m(s,i)-\gamma\theta_m(s',j)},
  \abs{\theta_m(s,i)-\gamma\theta_m(s',j)-1}
  \right\}>0,
  \label{eq:boundary_margin}
\end{equation*}
where $P^\pi(s^\prime\mid s)=\sum_{a\in\gA}\pi(a\mid s)P(s^\prime\mid s,a)$. Set
\begin{equation*}
  r_{\mathrm{out}}=
  \begin{cases}
  \displaystyle
  \min\left\{
  \frac{\Delta_m}{2(1+\gamma)},
  \frac{1}{8mC_0(1+\gamma)},
  \frac{c_\gM(1-\gamma)}{32mL(1+\gamma)^2}
  \right\}, & \text{under Assumption~\ref{assump:density}(i)},\\[1.2em]
  \displaystyle
  \min\left\{
  \frac{1}{8mC_0(1+\gamma)},
  \frac{c_\gM(1-\gamma)}{32mL(1+\gamma)^2}
  \right\}, & \text{under Assumption~\ref{assump:density}(ii)},
  \end{cases}
\end{equation*}
where a fraction with denominator $L=0$ is interpreted as $+\infty$.
Define the local contraction and nonlinear-remainder scales by
\begin{equation*}
  \mu=\frac{c_\gM(1-\gamma)}{4m},
  \qquad
  L_h=\frac{L(1+\gamma)^2}{2},
\end{equation*}
Since the step sizes are nonincreasing, write
$\bar\alpha=\sup_{t\ge0}\alpha_t=\alpha_0$ for their largest value.  Define
the global-drift constants by
\begin{equation*}
\label{eq:global_smooth_constants}
\begin{aligned}
  c_{\mathrm g}
  &=\frac{c_\gM\tau_1(1-\tau_1)(1-\gamma)r_{\mathrm{out}}}{8},\\
  \beta
  &=\frac{4}{(1-\gamma)r_{\mathrm{out}}}
  \log\left[\frac{2|\gS|m(1+2c_{\mathrm g})}{c_{\mathrm g}}\right],\\
  D_{\mathrm g}
  &=\frac{1+\bar\alpha}{1-\gamma}-\frac{r_{\mathrm{out}}}{2}
  +\frac{\log(2|\gS|m)}{\beta}.
\end{aligned}
\end{equation*}
These global-drift constants are used only in the entrance condition.
Local capture instead requires
$b_T(\delta)\le\min\{r_{\mathrm{out}}/16,\mu/(64L_h)\}$,
where $\mu/L_h=+\infty$ when $L_h=0$. For use in the proofs below, also
write $u_T=\lfloor T/4\rfloor$ and $v_T=\lfloor T/2\rfloor$, and define
\begin{equation*}
  A_T^{\mathrm{ent}}=\sum_{t=u_T}^{v_T-1}\alpha_t,
  \qquad
  V_T^{\mathrm{ent}}=\sum_{t=u_T}^{v_T-1}\alpha_t^2,
  \qquad
  A_T^{\mathrm{loc}}=\sum_{t=v_T}^{T-1}\alpha_t.
\end{equation*}
Also define
\begin{equation*}
  \Gamma_T
  =c_{\mathrm g}A_T^{\mathrm{ent}}-D_{\mathrm g}
  -\frac\beta2V_T^{\mathrm{ent}},
  \qquad
  b_T(\delta)
  =2\sqrt{\frac{\alpha_{u_T}\ell_T(\delta)}
  {c_\gM(1-\gamma)}}
  +\frac{\alpha_{u_T}\ell_T(\delta)}3.
\end{equation*}
For Corollary~\ref{cor:harmonic_stepsize}, one admissible universal choice is
\begin{equation*}
  C_{\mathrm{harm}}
  =\max\left\{8,\frac{16\log4}{15\log(4/3)}\right\}.
\end{equation*}
Then $c\ge C_{\mathrm{harm}}/\mu$ ensures both
$\exp\{-15\mu c\log(4/3)/16\}\le1/4$ and
$\mu c/6\ge4/3$. The latter inequality is used only in the sharper
harmonic entrance argument.

\subsection{Proof of Lemma~\ref{lem:global_smooth_max}: Global drift}
\begin{proof}
Put $V\coloneq\norm{\be}_\infty$ for the largest coordinate error. The bounds in
the standard log-sum-exp inequalities give
\begin{equation*}
 \norm{\be}_\infty\le\Phi_\beta(\be)
 \le\norm{\be}_\infty+\frac{\log(2|\gS|m)}{\beta}.
\end{equation*}
For
$(s,i)\in\gS\times[m]$ and $\varsigma\in\{-1,1\}$, define the softmax weight
\begin{equation*}
 w_{s,i,\varsigma}(\be)
 \coloneq
 \frac{\exp(\beta\varsigma e_{s,i})}
 {\sum_{(s',j)\in\gS\times[m]}
 \{\exp(\beta e_{s',j})+\exp(-\beta e_{s',j})\}}.
\end{equation*}
If $X$ is the auxiliary random signed basis vector that equals
$\varsigma\bm e_{s,i}$ with probability $w_{s,i,\varsigma}(\be)$, then
\begin{equation*}
 \nabla\Phi_\beta(\be)=\EB[X],
 \qquad
 \nabla^2\Phi_\beta(\be)=\beta\cov(X).
\end{equation*}
Since $\norm{X}_1=1$ and
$\abs{\bu^\top X}\le\norm{\bu}_\infty$, it follows that
\begin{equation*}
 \norm{\nabla\Phi_\beta(\be)}_1\le1,
 \qquad
 \bu^\top\nabla^2\Phi_\beta(\be)\bu
 =\beta\var(\bu^\top X)
 \le\beta\norm{\bu}_\infty^2.
\end{equation*}

It remains to prove the drift bound.  Let $\mathcal I(\be)$ be the set of
signed coordinates whose errors are close to the maximum:
\begin{equation*}
 \mathcal I(\be)\coloneq
 \left\{(s,i,\varsigma):\varsigma e_{s,i}
 \ge V-\frac{(1-\gamma)r_{\mathrm{out}}}{4}\right\}.
\end{equation*}
If $\varsigma=1$, then for every $(s^\prime,j)$,
\begin{equation*}
 e_{s,i}-\gamma e_{s^\prime,j}
 \ge(1-\gamma)V-\frac{(1-\gamma)r_{\mathrm{out}}}{4}
 \ge\frac{(1-\gamma)V}{2}.
\end{equation*}
We denote
\begin{equation*}
    V^*=\min\brc{V,\frac{1+\bar{\alpha}}{1-\gamma}}
\end{equation*}
and monotonicity of the reward CDFs and the definition of $c_\gM$ give
\begin{equation*}
\begin{aligned}
 h_{s,i}(\bm\theta)
 &\ge
 F_{(\gT^\pi\bm\eta_m)(s)}
 \left(\theta_m(s,i)
 +\frac{(1-\gamma)V^*}{2}\right)-\tau_i\\
 &\ge
 \frac{c_\gM\tau_1(1-\tau_1)
 (1-\gamma)V^*}{2}\coloneq a_mV^*
\end{aligned}
\end{equation*}
The same argument with all inequalities reversed shows that
$-h_{s,i}(\bm\theta)\ge a_mV^*$ when $\varsigma=-1$.

For a signed coordinate outside $\mathcal I(\be)$, its exponential weight is
at most $\exp\{-\beta(1-\gamma)r_{\mathrm{out}}/4\}$ times the weight of a
maximizing signed coordinate.  Hence
\begin{equation*}
 w\coloneq\sum_{(s,i,\varsigma)\notin\mathcal I(\be)}
 w_{s,i,\varsigma}(\be)
 \le 2|\gS|m
 \exp\left\{-\frac{\beta(1-\gamma)r_{\mathrm{out}}}{4}\right\}
 =\frac{c_{\mathrm g}}{1+2c_{\mathrm g}}.
\end{equation*}
Every signed mean-field coordinate has magnitude at most one.  The coordinates
in $\mathcal I(\be)$ contribute at least $a_mV^*$ per unit softmax
weight, whereas all remaining coordinates contribute at least $-1$.
Moreover, since $a_mV^*\geq 2c_\mathrm{g}$,
\begin{equation*}
    w\leq\frac{a_mV^*}{2(1+a_mV^*)}.
\end{equation*}
Therefore,
\begin{equation}\label{eq:stronger_global_drift}
 \left\langle\nabla\Phi_\beta(\be),\bh(\bm\theta)\right\rangle
 \ge a_mV^*(1-w)-w\ge \frac{a_mV^*}{2},
\end{equation}
Since $a_mV^*\geq 2c_\mathrm{g}$, we have actually proved a stronger version of the drift bound stated in Lemma~\ref{lem:global_smooth_max}. 
\end{proof}

\subsection{Proof of Lemma~\ref{lem:finite_time_entrance}: Late entrance}

\begin{proof}
Let $\be^{(t)}$ denote the error at time $t$, and fix the block endpoints
$u<v$ from Lemma~\ref{lem:finite_time_entrance}. Set
\begin{equation*}
    \be^{(t)}\coloneq\bm\theta^{(t)}-\bm\theta_m.
\end{equation*}
For integers $r<q$, write $A_{r,q}$ and $V_{r,q}$ for the accumulated step
size and accumulated squared step size, respectively:
\begin{equation*}
    A_{r,q}\coloneq\sum_{t=r}^{q-1}\alpha_t,
    \qquad
    V_{r,q}\coloneq\sum_{t=r}^{q-1}\alpha_t^2.
\end{equation*}
For the global comparison argument, use the drift scale $c_{\mathrm g}$ and
smoothing parameter $\beta$ defined in
Appendix~\ref{sec:explicit_proof_constants}.  We use the smooth maximum
\begin{equation*}
 \Phi_\beta(\be)
 \coloneq
 \frac1\beta\log
 \sum_{s\in\gS,\,i\in[m]}
 \left(e^{\beta e_{s,i}}+e^{-\beta e_{s,i}}\right).
\end{equation*}
Lemma~\ref{lem:global_smooth_max} gives
\begin{equation}\label{eq:smooth_max_equivalence}
    \norm{\be}_\infty
    \le\Phi_\beta(\be)
    \le\norm{\be}_\infty+\frac{\log(2|\gS|m)}{\beta},
\end{equation}
\begin{equation}\label{eq:smooth_max_hessian}
    \norm{\nabla\Phi_\beta(\be)}_1\le1,
    \qquad
    \bu^\top\nabla^2\Phi_\beta(\be)\bu
    \le\beta\norm{\bu}_\infty^2,
\end{equation}
and, whenever
$\norm{\bm\theta-\bm\theta_m}_\infty\ge r_{\mathrm{out}}/2$,
\begin{equation}\label{eq:global_smooth_drift}
    \left\langle
    \nabla\Phi_\beta(\bm\theta-\bm\theta_m),
    \bh(\bm\theta)
    \right\rangle
    \ge c_{\mathrm g}.
\end{equation}

For the sharper conditional form used in the Stage-I numerical diagnostic,
define the realized budget at the left endpoint of the block by
\begin{equation*}
    D_{u,\beta}
    \coloneq
    \Phi_\beta(\be^{(u)})-\frac{r_{\mathrm{out}}}{2}.
\end{equation*}
This quantity is $\gF_u$-measurable. 

For $t\ge u$, define the centered one-step fluctuation
\begin{equation*}
    Z_{t+1}
    \coloneq
    -\left\langle\nabla\Phi_\beta(\be^{(t)}),
    \bxi^{(t+1)}\right\rangle.
\end{equation*}
Then $\EB[Z_{t+1}\mid\gF_t]=0$ and $|Z_{t+1}|\le1$. Indeed, every
coordinate of $\bxi^{(t+1)}$ is the difference of two numbers in
$[0,1]$. Moreover,
\begin{equation*}
    \bh(\bm\theta^{(t)})+\bxi^{(t+1)}
    =\widehat\bh^{(t+1)}(\bm\theta^{(t)}),
    \qquad
    \norm{\widehat\bh^{(t+1)}(\bm\theta^{(t)})}_\infty\le1.
\end{equation*}
Thus, on $\{t<\tau_{u,v}\}$, Taylor's theorem, the QTD recursion, and
Equations~\eqref{eq:smooth_max_hessian}--\eqref{eq:global_smooth_drift}
imply
\begin{equation}\label{eq:one_step_global_drift}
    \Phi_\beta(\be^{(t+1)})
    \le
    \Phi_\beta(\be^{(t)})
    -c_{\mathrm g}\alpha_t
    +\alpha_tZ_{t+1}
    +\frac\beta2\alpha_t^2.
\end{equation}

Define the corresponding stopped process by
\begin{equation*}
    M_k\coloneq
    \sum_{t=u}^{k-1}\alpha_t\ind\{t<\tau_{u,v}\}Z_{t+1},
    \qquad u\le k\le v,
\end{equation*}
which is a martingale because $\ind\{t<\tau_{u,v}\}$ is
$\gF_t$-measurable; its increments are bounded by $\alpha_t$.
On $\{\tau_{u,v}=\infty\}$, summing
Equation~\eqref{eq:one_step_global_drift} from $u$ to $v-1$ and using
$\Phi_\beta(\be^{(v)})>r_{\mathrm{out}}/2$ yields
\begin{equation*}
    M_v
    \ge
    c_{\mathrm g}A_{u,v}
    -D_{u,\beta}
    -\frac\beta2V_{u,v}.
\end{equation*}
Conditionally on $\gF_u$, the Azuma--Hoeffding inequality therefore gives
\begin{equation}\label{eq:initialization_specific_entrance}
    \PB(\tau_{u,v}=\infty\mid\gF_u)
    \le
    \exp\left\{-
    \frac{\left(c_{\mathrm g}A_{u,v}-D_{u,\beta}
    -\frac\beta2V_{u,v}\right)_+^2}{2V_{u,v}}
    \right\}.
\end{equation}
By Lemma~\ref{lem:pathwise_boundedness} and Lemma~\ref{lem:global_smooth_max}, we know that $D_{u,\beta}\le D_{\mathrm g}$.
Therefore, the right-hand side of
Equation~\eqref{eq:initialization_specific_entrance} is at most the same
expression with $D_{u,\beta}$ replaced by $D_{\mathrm g}$.  Taking
expectations with respect to $\gF_u$ therefore yields
\begin{equation*}
    \PB(\tau_{u,v}=\infty)
    \le
    \exp\left\{-
    \frac{\left(c_{\mathrm g}A_{u,v}-D_{\mathrm g}
    -\frac\beta2V_{u,v}\right)_+^2}{2V_{u,v}}
    \right\}.
\end{equation*}
This is the claimed initialization-uniform entrance bound.  In particular, the
entrance condition~\eqref{eq:entrance_condition} of
Theorem~\ref{thm:global_nonasymptotic} gives directly
\begin{equation*}
  c_{\mathrm g}A_T^{\mathrm{ent}}-D_{\mathrm g}
  -\frac\beta2V_T^{\mathrm{ent}}
  \ge\sqrt{2V_T^{\mathrm{ent}}\log\frac2\delta}.
\end{equation*}
Thus, for $(u,v)=(u_T,v_T)$, the right-hand side is at most $\delta/2$.

For completeness, Stage~I also admits a fixed-$m$ refinement.  Define
\begin{equation}
 c_{\gM,m}\coloneq\inf_{\substack{s\in\gS,\ i\in[m]\\0<|z|\le(1-\gamma)^{-1}}}\left|\frac{F_{(\gT^\pi\bm\eta_m)(s)}(\theta_m(s,i)+z)-\tau_i}{\tau_i(1-\tau_i)z}\right|.\label{eq:fixed_m_identifiability}
\end{equation}
Then $c_{\gM,m}\ge c_\gM$.  The preceding entrance proof only invokes the identifiability inequality at the current value of $m$, so the same conditional bound remains valid when the Stage-I constants $r_{\mathrm{out}}$, $c_{\mathrm g}$, and $\beta$ are computed with $c_{\gM,m}$.  The main theorem retains the smaller, uniform constant $c_\gM$ to use one MDP-dependent sensitivity constant across quantile resolutions.
\end{proof}

\subsection{Proof of Lemma~\ref{lem:jacobian_summarize}: Local linear structure}
\label{sec:local_matrix_entries}

For clarity, the matrix in the local derivative
$\bG_m=\bD_m-\gamma\bB_m$ has entries
\begin{equation*}
 (\bB_m)_{(s,i),(s',j)}
 =\frac1m\sum_{a\in\gA}\pi(a\mid s)P(s'\mid s,a)
 p_{s,a}\bigl(\theta_m(s,i)-\gamma\theta_m(s',j)\bigr).
\end{equation*}
Here $\bD_m=\diag_{s,i}(d_{s,i})$ and
$d_{s,i}=p_{(\gT^\pi\bm\eta_m)(s)}(\theta_m(s,i))$.

\begin{proof}
The boundary argument in the proof of
Lemma~\ref{lem:local_density_stability} shows that, throughout the
$r_{\mathrm{out}}$-ball, every density argument either remains in
$(0,1)$ or remains in one of the two exterior components. Under
condition~{\rm (ii)} of Assumption~\ref{assump:density}, the zero
extension is Lipschitz across the endpoints as well. Hence $\bh$ is
continuously differentiable on this ball. Differentiating
Equation~\eqref{eq:qtd_mean_field} at $\bm\theta_m$ gives
\begin{equation*}
    \nabla\bh(\bm\theta_m)=\bD_m-\gamma\bB_m=\bG_m.
\end{equation*}
Summing a row of $\bB_m$ over $(s^\prime,j)$ and using
Equation~\eqref{eq:bellman_target_density_formula} gives
$\bB_m\bm1=\bD_m\bm1$. This proves the first part of
Lemma~\ref{lem:jacobian_summarize}.

Equation~\eqref{eq:technical_density_lower_from_cM} gives
$d_{s,i}\ge c_\gM/(4m)$.
Since $d_{s,i}\le C_0$, the condition $0\le\alpha C_0\le1$ implies
$\bI-\alpha\bG_m\ge\bm0$. Its row sums satisfy
\begin{equation*}
    (\bI-\alpha\bG_m)\bm1=\bm1-\alpha(1-\gamma)\bD_m\bm1\le
    \left(1-\frac{\alpha c_\gM(1-\gamma)}{4m}\right)\bm1.
\end{equation*}
For a nonnegative matrix, the largest row sum is its induced
$\ell_\infty$ norm. Consequently,
\begin{equation*}
    \bI-\alpha\bG_m\ge\bm0,
    \qquad
    \norm{\bI-\alpha\bG_m}_\infty
    \le1-\frac{\alpha c_\gM(1-\gamma)}{4m}=1-\alpha\mu,
\end{equation*}
which proves the second part of Lemma~\ref{lem:jacobian_summarize}.
For completeness, the same argument also gives the additional structural
fact that $\bG_m$ is a nonsingular $M$-matrix. Choose any
$\alpha\in(0,C_0^{-1}]$.  The preceding bound gives
$\norm{\bI-\alpha\bG_m}_\infty<1$, so the Neumann series converges and
\begin{equation*}
 \bG_m^{-1}
 =\alpha\sum_{k=0}^\infty(\bI-\alpha\bG_m)^k\ge\bm0.
\end{equation*}
Moreover, $\bG_m=\bD_m-\gamma\bB_m$ has nonpositive off-diagonal entries.
Thus it is a $Z$-matrix with a nonnegative inverse, equivalently a
nonsingular $M$-matrix.

Finally, fix one density argument $x$ at the fixed point and let
$w=e_{s,i}-\gamma e_{s^\prime,j}$ be its scalar perturbation. The
$L$-Lipschitz property gives
\begin{equation*}
    \abs{F_{s,a}(x+w)-F_{s,a}(x)-p_{s,a}(x)w}
    \le\frac L2w^2.
\end{equation*}
Averaging this inequality and using
$\abs{w}\le(1+\gamma)\norm{\be}_\infty$ proves the final claim of
Lemma~\ref{lem:jacobian_summarize}.
\end{proof}

\subsection{Proof of Lemma~\ref{lem:variance_drift_matching}: Variance--drift matching}
\label{sec:proof_variance_drift_matching}
\begin{proof}
Since $\bB_m\bm1=\bD_m\bm1$,
\begin{equation*}
\bv^\top\bm1-\bv^\top(\bI-\alpha_t\bG_m)\bm1
=\alpha_t\bv^\top\bG_m\bm1
=(1-\gamma)\alpha_t\bv^\top\bD_m\bm1,
\end{equation*}
which proves Equation~\eqref{eq:drift_matching}.

For the variance bound, let $Y_{s,i}^{(t+1)}$ be the random empirical-CDF
value used to update the current coordinate:
\begin{equation*}
    Y_{s,i}^{(t+1)}
    \coloneq\frac1m\sum_{j=1}^m
    \ind\left\{r^{(t+1,s)}+\gamma\theta^{(t)}
    \bigl(s^{\prime(t+1,s)},j\bigr)<\theta^{(t)}(s,i)\right\}.
\end{equation*}
Its conditional mean is denoted by
$q_{s,i}^{(t)}\coloneq\EB[Y_{s,i}^{(t+1)}\mid\gF_t]$.
Since $Y_{s,i}^{(t+1)}\in[0,1]$, its conditional variance is at most
$q_{s,i}^{(t)}(1-q_{s,i}^{(t)})$.
Assumption~\ref{assump:density} implies that the density entering
$q_{s,i}^{(t)}$ is bounded by $C_0$, so the definition of
$r_{\mathrm{out}}$ gives
\begin{equation*}
\begin{aligned}
    \abs{q_{s,i}^{(t)}-\tau_i}
    &\le C_0(1+\gamma)
    \norm{\bm\theta^{(t)}-\bm\theta_m}_\infty\\
    &\le\frac1{8m}
    \le\frac12\tau_i(1-\tau_i).
\end{aligned}
\end{equation*}
Consequently,
\begin{equation}\label{eq:coordinate_variance_bound}
\begin{aligned}
    \EB[(\xi_{s,i}^{(t+1)})^2\mid\gF_t]
    &\le q_{s,i}^{(t)}(1-q_{s,i}^{(t)})\\
    &=\tau_i(1-\tau_i)
      +(q_{s,i}^{(t)}-\tau_i)(1-2\tau_i)
      -(q_{s,i}^{(t)}-\tau_i)^2\\
    &\le\frac32\tau_i(1-\tau_i)
    \le2c_\gM^{-1}d_{s,i},
\end{aligned}
\end{equation}
where the last inequality follows from
Equation~\eqref{eq:technical_density_lower_from_cM}. For every
$\bv\ge\bm0$, Cauchy--Schwarz inequality gives
\begin{equation*}
    (\bv^\top\bxi^{(t+1)})^2
    \le(\bv^\top\bm1)
    \sum_{s,i}v_{s,i}(\xi_{s,i}^{(t+1)})^2.
\end{equation*}
Taking conditional expectations and applying
Equation~\eqref{eq:coordinate_variance_bound} proves
Equation~\eqref{eq:variance_matching}.
\end{proof}

\subsection{Proof of Lemma~\ref{lem:variance_sensitive_local}: Local capture}
\label{sec:proof_variance_sensitive_local}

\begin{proof}
Condition on $\gF_n$. Let $\be^{(t)}$ be the local error and let $\sigma$
be the first exit time from the $r_{\mathrm{out}}$-ball:
\begin{equation*}
    \be^{(t)}\coloneq\bm\theta^{(t)}-\bm\theta_m,
    \qquad
    \sigma\coloneq
    \inf\left\{t\ge n:\norm{\be^{(t)}}_\infty>r_{\mathrm{out}}\right\}.
\end{equation*}
The constants $\mu$ and $L_h$ below are, respectively, the local contraction
scale and the coefficient of the quadratic remainder:
\begin{equation*}
    \mu\coloneq\frac{c_\gM(1-\gamma)}{4m},
    \qquad
    L_h\coloneq\frac{L(1+\gamma)^2}{2}.
\end{equation*}
For $r\le k$, let $\bPhi_{r,k}$ denote the linearized propagator from time
$r$ to time $k$:
\begin{equation*}
    \bPhi_{r,k}\coloneq\prod_{j=r}^{k-1}(\bI-\alpha_j\bG_m).
\end{equation*}
The bold symbol $\bPhi_{r,k}$ is a matrix propagator and is distinct from the
scalar smooth maximum $\Phi_\beta$ used in the entrance argument.
Under the standing assumptions, $\bar\alpha C_0\le1$ and
$\alpha_t\le\bar\alpha$.
Moreover,
$\mu\le c_\gM/(4m)\le d_{s,i}\le C_0$, where we used
Equation~\eqref{eq:technical_density_lower_from_cM}. Hence
$0\le\mu\alpha_t\le1$. The bounds proved in
Lemma~\ref{lem:jacobian_summarize} give
\begin{equation*}
    \bI-\alpha_t\bG_m\ge\bm0,
    \qquad
    \norm{\bI-\alpha_t\bG_m}_\infty\le1-\mu\alpha_t.
\end{equation*}

For $t<\sigma$, Equation~\eqref{eq:qtd_decomposition} gives
\begin{equation*}
    \be^{(t+1)}
    =(\bI-\alpha_t\bG_m)\be^{(t)}
    -\alpha_t\bxi^{(t+1)}
    -\alpha_t\bR(\be^{(t)}).
\end{equation*}
Consequently, for every $n<k\le\sigma$,
\begin{equation*}
    \be^{(k)}=\bPhi_{n,k}\be^{(n)}-\sum_{t=n}^{k-1}\alpha_t\bPhi_{t+1,k}\bxi^{(t+1)}-\sum_{t=n}^{k-1}\alpha_t\bPhi_{t+1,k}\bR(\be^{(t)}).
\end{equation*}

Define the local logarithmic complexity $\ell_T^{\mathrm{loc}}(\delta)$ and the resulting uniform
noise envelope $b_{n,T}(\delta)$ by
\begin{equation*}
\begin{aligned}
    \ell_T^{\mathrm{loc}}(\delta)
    &\coloneq\log\frac{4|\gS|mT^2}{\delta},\\
    b_{n,T}(\delta)
    &\coloneq
    2\sqrt{\frac{\alpha_n\ell_T^{\mathrm{loc}}(\delta)}
    {c_\gM(1-\gamma)}}
    +\frac{\alpha_n\ell_T^{\mathrm{loc}}(\delta)}{3}.
\end{aligned}
\end{equation*}
Lemma~\ref{lem:uniform_linearized_martingale} implies that, with conditional
probability at least $1-\delta$,
\begin{equation}\label{eq:uniform_noise_event}
    \max_{n\le r<k\le T}
    \norm{
    \sum_{t=r}^{k-1}\alpha_t
    \bPhi_{t+1,k}
    \ind\{t<\sigma\}\bxi^{(t+1)}
    }_\infty
    \le b_{n,T}(\delta).
\end{equation}

The definition of $r_{\mathrm{out}}$ gives
$L_hr_{\mathrm{out}}/\mu\le1/16$, and the other two smallness conditions in
Lemma~\ref{lem:nonlinear_epoch_bootstrap} are exactly the assumed bounds on
$b_{n,T}(\delta)$.

We now work on the event in Equation~\eqref{eq:uniform_noise_event}.
Applying Lemma~\ref{lem:nonlinear_epoch_bootstrap} gives
\begin{equation*}
    \norm{\be^{(T)}}_\infty
    \le
    \norm{\be^{(n)}}_\infty
    \exp\left\{-\frac{15\mu}{16}
    \sum_{t=n}^{T-1}\alpha_t\right\}
    +b_{n,T}(\delta)
    +\frac{32L_h}{15\mu}b_{n,T}^2(\delta).
\end{equation*}
The same bootstrap also proves that the exit time satisfies $\sigma>T$.
This is the claimed local-capture statement and completes the proof.
\end{proof}

\subsection{Proofs for Section~\ref{Section:analysis}}
\label{sec:proof_stepsize_specializations}

\begin{proof}[Proof of Theorem~\ref{thm:global_nonasymptotic}]
Apply Lemma~\ref{lem:finite_time_entrance} with
$u=u_T$ and $v=v_T$. The entrance condition gives
$\Gamma_T\ge\sqrt{2V_T^{\mathrm{ent}}\log(2/\delta)}$, so
\begin{equation*}
  \PB(\tau_{u_T,v_T}=\infty)\le\frac\delta2.
\end{equation*}
On the complementary event, set $n=\tau_{u_T,v_T}$.  This is an
$(\gF_t)$-stopping time, $u_T\le n\le v_T$, and
$\norm{\bm\theta^{(n)}-\bm\theta_m}_\infty\le r_{\mathrm{out}}/2$.
Because the step sizes are nonincreasing,
\begin{equation*}
  \alpha_n\le\alpha_{u_T},\qquad
  \sum_{t=n}^{T-1}\alpha_t\ge
  \sum_{t=v_T}^{T-1}\alpha_t=A_T^{\mathrm{loc}}.
\end{equation*}
Moreover, the local envelope in Lemma~\ref{lem:variance_sensitive_local}
satisfies $b_{n,T}(\delta/2)\le b_T(\delta)$, because
$\ell_T^{\mathrm{loc}}(\delta/2)=\ell_T(\delta)$ and
$\alpha_n\le\alpha_{u_T}$. The capture
condition~\eqref{eq:capture_condition} therefore lets us apply that lemma. To make the
random starting time explicit, for each
$t\in\{u_T,\ldots,v_T\}$ let $\gE_t$ denote the local-capture conclusion
when the lemma is started at time $t$, and define $\gE_{\mathrm{loc}}$ on
$\{n<\infty\}$ by $\gE_{\mathrm{loc}}=\gE_n$.  The conditional conclusion
of the lemma gives
\begin{equation*}
 \PB(\gE_t^c\mid\gF_t)\ind\{n=t\}
 \le\frac\delta2\ind\{n=t\}.
\end{equation*}
Taking expectations and summing over the disjoint events $\{n=t\}$ shows
\begin{equation*}
 \PB(\gE_{\mathrm{loc}}^c,\ n<\infty)
 \le\sum_{t=u_T}^{v_T}
 \EB\left[\ind\{n=t\}
 \PB(\gE_t^c\mid\gF_t)\right]
 \le\frac\delta2.
\end{equation*}
On $\gE_{\mathrm{loc}}$, the lemma yields
\begin{equation*}
  \norm{\bm\theta^{(T)}-\bm\theta_m}_\infty
  \le \frac{r_{\mathrm{out}}}{2}
  \exp\left\{-\frac{15\mu}{16}A_T^{\mathrm{loc}}\right\}
  +b_T(\delta)+\frac{32L_h}{15\mu}b_T^2(\delta).
\end{equation*}
The assumed inequality $b_T(\delta)\le\mu/(64L_h)$ bounds the quadratic
remainder when $L_h>0$ because
\begin{equation*}
 \frac{32L_h}{15\mu}b_T^2(\delta)
 \le\frac{32}{15\cdot64}b_T(\delta)=\frac1{30}b_T(\delta).
\end{equation*}
The remainder vanishes when $L_h=0$.  Thus, on the same event,
\begin{equation}
  \norm{\bm\theta^{(T)}-\bm\theta_m}_\infty
  \le \frac{r_{\mathrm{out}}}{2}
  \exp\left\{-\frac{15\mu}{16}A_T^{\mathrm{loc}}\right\}
  +\frac{31}{30}b_T(\delta).
  \label{eq:generic_hp_bound_explicit}
\end{equation}
A union bound gives this inequality with probability at least $1-\delta$.
This is exactly Equation~\eqref{eq:generic_hp_bound}.
\end{proof}

\begin{proof}[Proof of Corollary~\ref{cor:constant_polynomial}]
For the constant schedule, $\alpha_{u_T}=\eta$ and
$A_T^{\mathrm{loc}}\ge\eta T/3$ for $T\ge8$.  Substitution into
Equation~\eqref{eq:generic_hp_bound_explicit} gives the sharper appendix
bound
\begin{equation*}
  \norm{\bm\theta^{(T)}-\bm\theta_m}_\infty
  \le \frac{r_{\mathrm{out}}}{2}e^{-5\mu\eta T/16}
  +\frac{31}{15}\sqrt{\frac{\eta\ell_T(\delta)}
  {c_\gM(1-\gamma)}}+\frac{31}{90}\eta\ell_T(\delta),
\end{equation*}
which implies the first claim.

Now let $\alpha_t=c(t+1)^{-a}$ with $a\in(0,1)$.  For $T\ge8$,
monotonicity and the block lengths give the explicit bounds
\begin{equation*}
  \alpha_{u_T}\le4cT^{-a},\qquad
  A_T^{\mathrm{loc}}\ge\frac c2T^{1-a}.
\end{equation*}
Substituting these inequalities into
Equation~\eqref{eq:generic_hp_bound_explicit} yields
\begin{equation*}
\begin{aligned}
  \norm{\bm\theta^{(T)}-\bm\theta_m}_\infty
  \le{}&\frac{r_{\mathrm{out}}}{2}
  \exp\left\{-\frac{15\mu c}{32}T^{1-a}\right\}
  +\frac{62}{15}\sqrt{\frac{c\ell_T(\delta)}
  {c_\gM(1-\gamma)T^a}}
  +\frac{62c\ell_T(\delta)}{45T^a}.
\end{aligned}
\end{equation*}
Because $\mu=c_\gM(1-\gamma)/(4m)$, this explicit inequality proves
Equation~\eqref{eq:polynomial_hp_bound}.

For completeness, the burn-in conditions are eventually satisfied for every
fixed $\delta$.  Indeed, on the entrance block,
\begin{equation*}
  \frac c8T^{1-a}\le A_T^{\mathrm{ent}}\le2cT^{1-a},
  \qquad
  V_T^{\mathrm{ent}}\le8c^2T^{1-2a}.
\end{equation*}
Consequently,
$V_T^{\mathrm{ent}}/A_T^{\mathrm{ent}}\le64cT^{-a}$,
$(A_T^{\mathrm{ent}})^2/V_T^{\mathrm{ent}}\ge T/512$, and
$b_T(\delta)\to0$.  Indeed, because
$A_T^{\mathrm{ent}}\to\infty$ and
$V_T^{\mathrm{ent}}/A_T^{\mathrm{ent}}\to0$, for all sufficiently large
$T$ we have
\begin{equation*}
 D_{\mathrm g}\le\frac{c_{\mathrm g}}4A_T^{\mathrm{ent}},
 \qquad
 \frac\beta2V_T^{\mathrm{ent}}
 \le\frac{c_{\mathrm g}}4A_T^{\mathrm{ent}}.
\end{equation*}
Thus $\Gamma_T\ge c_{\mathrm g}A_T^{\mathrm{ent}}/2$, and
\begin{equation*}
 \frac{\Gamma_T^2}{2V_T^{\mathrm{ent}}}
 \ge\frac{c_{\mathrm g}^2(A_T^{\mathrm{ent}})^2}
 {8V_T^{\mathrm{ent}}}\longrightarrow\infty.
\end{equation*}
This verifies the entrance condition for every fixed $\delta$, while
$b_T(\delta)\to0$ verifies the local condition.  Hence the argument covers
the full range $a\in(0,1)$.
\end{proof}

\paragraph{Dependence of the sufficient burn-in on quantile resolution.}
Under Assumption~\ref{assump:density}(ii), keep
$c_\gM,C_0,L,\gamma$ and the polynomial-schedule parameters $a,c$ fixed.
The explicit constants satisfy
\begin{equation*}
 r_{\mathrm{out}}\asymp m^{-1},\qquad
 \mu\asymp m^{-1},\qquad
 c_{\mathrm g}\asymp m^{-2},\qquad
 \beta=\widetilde\Theta(m),\qquad D_{\mathrm g}=O(1).
\end{equation*}
For $\alpha_t=c(t+1)^{-a}$, the entrance-block sums have orders
$A_T^{\mathrm{ent}}\asymp T^{1-a}$ and
$V_T^{\mathrm{ent}}\asymp T^{1-2a}$. Separately dominating
$D_{\mathrm g}$, the potential-curvature cost
$\beta V_T^{\mathrm{ent}}/2$, and the martingale deviation in
Equation~\eqref{eq:entrance_condition} gives, up to logarithmic factors,
\begin{equation*}
 T^{1-a}\gtrsim m^2,\qquad T^a\gtrsim m^3,\qquad T\gtrsim m^4.
\end{equation*}
The capture condition itself only requires $T^a\gtrsim m^2$, up to
logarithmic factors. Consequently, a sufficient burn-in length can be
chosen of order
\begin{equation*}
 \widetilde O\!\left(
 \max\{m^{2/(1-a)},\ m^{3/a},\ m^4\}\right).
\end{equation*}
Here logarithmic factors may depend on the failure probability $\delta$ and the number
of coordinates $|\gS|m$, and the implicit constants may depend on the fixed
schedule and problem parameters. The same polynomial scaling holds under
Assumption~\ref{assump:density}(i) if $\Delta_m\gtrsim m^{-1}$.
If this boundary margin is smaller, it also enters the local radius and
the global-drift constants, so the displayed scaling need not apply.

\begin{proof}[Proof of Corollary~\ref{cor:harmonic_stepsize}]
Write $\alpha_t=c/(t+t_0)$. The offset condition gives
$\bar\alpha C_0=(c/t_0)C_0\le1$, and the gain condition gives
\begin{equation*}
  q\coloneq\exp\{-15\mu c\log(4/3)/16\}\le\frac14.
\end{equation*}
We sharpen the entrance step by retaining the dependence of the global drift
on the current error magnitude. Write
\begin{equation*}
    \be^{(t)}\coloneq\bm\theta^{(t)}-\bm\theta_m.
\end{equation*}
Set
\begin{equation*}
    \lambda_{\mathrm h}\coloneq\frac{\mu}{6},
    \qquad
    p\coloneq\lambda_{\mathrm h}c.
\end{equation*}
By the choice of $C_{\mathrm{harm}}$ in
Appendix~\ref{sec:explicit_proof_constants},
$p\ge4/3>1$.

Also set
\begin{equation*}
    B_\Phi\coloneq\frac{1+\bar\alpha}{1-\gamma}+\frac{\log(2|\gS|m)}{\beta}.
\end{equation*}
Lemma~\ref{lem:pathwise_boundedness} and the smooth-max equivalence imply
$\Phi_\beta(\be^{(t)})\le B_\Phi$ for every $t$.

Choose an integer $u\ge\max\{1,t_0\}$ and define
\begin{equation*}
    v=\left\lceil (u+t_0)^{4/3}-t_0\right\rceil,
    \qquad N=2v.
\end{equation*}
We will increase $u$ below until a finite list of deterministic inequalities
is satisfied. Let
\begin{equation*}
    \tau=\inf\left\{u\le t\le v:\norm{\be^{(t)}}_\infty\le\frac{r_{\mathrm{out}}}{2}\right\},
\end{equation*}
with the convention $\inf\varnothing=\infty$, and define
\begin{equation*}
    Z_{t+1}=-\left\langle\nabla\Phi_\beta(\be^{(t)}),\bxi^{(t+1)}\right\rangle.
\end{equation*}
Then
$\EB[Z_{t+1}\mid\gF_t]=0$ and $|Z_{t+1}|\le1$.
On $\{t<\tau\}$, Equation~\eqref{eq:stronger_global_drift} yields
\begin{equation*}
  \left\langle\nabla\Phi_\beta(\be),\bh(\bm\theta)\right\rangle
  \ge \frac{\mu}{4}\norm{\be}_\infty, 
\end{equation*}
and the definition of $\beta$ gives
\begin{equation*}
 \frac{\log(2|\gS|m)}{\beta}\le\frac{(1-\gamma)r_{\mathrm{out}}}{4}\le\frac{V}{2}.
\end{equation*}
Thus $\Phi_\beta(\be)\le3V/2$ and
\begin{equation*}
  \left\langle\nabla\Phi_\beta(\be),\bh(\bm\theta)\right\rangle\ge \frac{\mu}{6}\Phi_\beta(\be).
\end{equation*}
This inequality and the same Taylor expansion
as in the proof of Lemma~\ref{lem:finite_time_entrance} give
\begin{equation}\label{eq:harmonic_multiplicative_drift}
    \Phi_\beta(\be^{(t+1)})\le(1-\lambda_{\mathrm h}\alpha_t)\Phi_\beta(\be^{(t)})+\alpha_tZ_{t+1}+\frac\beta2\alpha_t^2.
\end{equation}

For $u\le t<v$, define the deterministic backward weights
\begin{equation*}
    \omega_{t,v}\coloneq\prod_{j=t+1}^{v-1}(1-\lambda_{\mathrm h}\alpha_j),
    \qquad
    \omega_{u-1,v}\coloneq\prod_{j=u}^{v-1}(1-\lambda_{\mathrm h}\alpha_j).
\end{equation*}
Since $0\le\lambda_{\mathrm h}\alpha_j\le1$,
$0\le\omega_{t,v}\le1$.
Integral comparison gives
\begin{equation}
\begin{aligned}
    \omega_{u-1,v}&\le\exp\left\{-\lambda_{\mathrm h}\sum_{t=u}^{v-1}\alpha_t\right\}\\
    &\le\left(\frac{u+t_0}{v+t_0}\right)^p\le(u+t_0)^{-p/3}.
\end{aligned}
\label{eq:harmonic_initial_weight}
\end{equation}

We also need a weighted squared-step estimate. For
$k\in\{u,\ldots,v\}$, set
\begin{equation*}
    H_k\coloneq\sum_{t=u}^{k-1}\alpha_t^2\prod_{j=t+1}^{k-1}(1-\lambda_{\mathrm h}\alpha_j).
\end{equation*}
Then $H_u=0$ and, writing $x=k+t_0$,
\begin{equation*}
    H_{k+1}=\left(1-\frac{p}{x}\right)H_k+\frac{c^2}{x^2}.
\end{equation*}
An induction therefore yields
\begin{equation}\label{eq:harmonic_weighted_square_sum}
    H_k\le\frac{c^2}{(p-1)(k+t_0)},
    \qquad u\le k\le v.
\end{equation}
Indeed, if the bound holds at $k$, then
\begin{equation*}
    H_{k+1}\le\frac{c^2}{p-1}\frac{x-1}{x^2}\le\frac{c^2}{(p-1)(x+1)}.
\end{equation*}

Because $0\le\omega_{t,v}\le1$,
Equation~\eqref{eq:harmonic_weighted_square_sum} also gives
\begin{equation*}
    \sum_{t=u}^{v-1}\alpha_t^2\omega_{t,v}^2\le\frac{c^2}{(p-1)(v+t_0)}.
\end{equation*}

Now consider the stopped martingale
\begin{equation*}
    M_v\coloneq\sum_{t=u}^{v-1}\alpha_t\omega_{t,v}\ind\{t<\tau\}Z_{t+1}.
\end{equation*}
Azuma--Hoeffding and the preceding squared-weight bound imply
\begin{equation}\label{eq:harmonic_weighted_azuma}
    \PB\left(M_v>c\sqrt{\frac{2\log(4/\delta)}{(p-1)(v+t_0)}}\right)\le\frac\delta4.
\end{equation}

On $\{\tau=\infty\}$, unrolling
Equation~\eqref{eq:harmonic_multiplicative_drift} gives
\begin{equation*}
\begin{aligned}
 \Phi_\beta(\be^{(v)})\le&B_\Phi\omega_{u-1,v}+M_v+\frac\beta2\sum_{t=u}^{v-1}\alpha_t^2\omega_{t,v}\\
 \le&B_\Phi(u+t_0)^{-p/3}+M_v+\frac{\beta c^2}{2(p-1)(v+t_0)}.
\end{aligned}
\end{equation*}

Increase $u$, if necessary, until
\begin{equation}
\begin{aligned}
 B_\Phi(u+t_0)^{-p/3}&\le\frac{r_{\mathrm{out}}}{8},\\
 \frac{\beta c^2}{2(p-1)(v+t_0)}&\le\frac{r_{\mathrm{out}}}{8},\\
 c\sqrt{\frac{2\log(4/\delta)}{(p-1)(v+t_0)}}&\le\frac{r_{\mathrm{out}}}{8}.
\end{aligned}
\label{eq:harmonic_scale_sensitive_entrance_conditions}
\end{equation}
All three conditions hold for all sufficiently large $u$, because
$v+t_0\ge(u+t_0)^{4/3}$.

On the event complementary to that in
Equation~\eqref{eq:harmonic_weighted_azuma},
the event $\{\tau=\infty\}$ would imply
\begin{equation*}
 \Phi_\beta(\be^{(v)})\le\frac{3r_{\mathrm{out}}}{8}<\frac{r_{\mathrm{out}}}{2},
\end{equation*}
contradicting
\[
 \Phi_\beta(\be^{(v)})\ge\norm{\be^{(v)}}_\infty>\frac{r_{\mathrm{out}}}{2}.
\]
Hence
$\PB(\tau=\infty)\le\delta/4$.

On the complementary event,
$n=\tau\in[u,v]$ is a stopping time satisfying
\begin{equation*}
 \norm{\bm\theta^{(n)}-\bm\theta_m}_\infty
 \le
 \frac{r_{\mathrm{out}}}{2}.
\end{equation*}

The interval from entrance to $N$ leaves enough time for a fixed local
reduction: since $v\ge u\ge t_0$,
\begin{equation*}
 \sum_{t=v}^{N-1}\alpha_t
 \ge c\log\frac{2v+t_0}{v+t_0}
 \ge c\log(4/3).
\end{equation*}
Moreover, $\alpha_n\le c/(u+t_0)$, so the initial local envelope satisfies
\begin{equation*}
\begin{aligned}
 b_{n,N}(\delta/4)\le\bar b^{\mathrm{init}}_u
 \coloneq{}&2\sqrt{\frac{c}{c_\gM(1-\gamma)(u+t_0)}
              \log\frac{16|\gS|mN^2}{\delta}}\\
 &+\frac{c}{3(u+t_0)}\log\frac{16|\gS|mN^2}{\delta}.
\end{aligned}
\end{equation*}
Since $N=O((u+t_0)^{4/3})$, we have
$\bar b^{\mathrm{init}}_u\to0$. Increase $u$, if necessary, until
\begin{equation*}
  \bar b^{\mathrm{init}}_u\le
  \min\left\{\frac{r_{\mathrm{out}}}{16},
             \frac{\mu}{64L_h}\right\},
\end{equation*}
where the second entry is $+\infty$ when $L_h=0$.
Lemma~\ref{lem:variance_sensitive_local}, in the exact form supplied by
Lemma~\ref{lem:nonlinear_epoch_bootstrap}, now gives
\begin{equation*}
\begin{aligned}
 \norm{\bm\theta^{(N)}-\bm\theta_m}_\infty
 &\le q\frac{r_{\mathrm{out}}}{2}
      +b_{n,N}(\delta/4)
      +\frac{32L_h}{15\mu}b_{n,N}^2(\delta/4)\\
 &\le\frac{r_{\mathrm{out}}}{8}
      +\frac{31r_{\mathrm{out}}}{480}
 <\frac{r_{\mathrm{out}}}{2}.
\end{aligned}
\end{equation*}
The quadratic term is at most $b_{n,N}(\delta/4)/30$ when $L_h>0$ and
vanishes when $L_h=0$. The same lemma keeps the trajectory inside the
$r_{\mathrm{out}}$-ball through time $N$. Its failure probability,
conditionally on $\gF_n$, is at most $\delta/4$. Summing these conditional
bounds over the disjoint events $\{n=t\}$ gives an unconditional failure
probability at most $\delta/4$, just as in the proof of
Theorem~\ref{thm:global_nonasymptotic}. Thus entrance and initial capture
together cost at most $\delta/2$.


It remains to propagate the rate beyond $N$. Let $N_j=2^jN$ and assign
failure probabilities $\delta_j=\delta/2^{j+3}$, so that
$\sum_{j\ge0}\delta_j=\delta/4$. On the full block
$[N_j,N_{j+1}]$, define the exact local envelope
\begin{equation*}
 b_j=2\sqrt{\frac{\alpha_{N_j}}{c_\gM(1-\gamma)}
 \log\frac{4|\gS|mN_{j+1}^2}{\delta_j}}
 +\frac{\alpha_{N_j}}3
 \log\frac{4|\gS|mN_{j+1}^2}{\delta_j}.
\end{equation*}
Since $N_j\ge N\ge t_0$, integral comparison and monotonicity give
\begin{equation*}
 \sum_{t=N_j}^{N_{j+1}-1}\alpha_t\ge c\log(4/3),
 \qquad \alpha_{N_j}\le\frac{c}{N_j}.
\end{equation*}
Also, $2^j=N_j/N\le N_j$, and hence
\begin{equation*}
 \log\frac{4|\gS|mN_{j+1}^2}{\delta_j}
 =\log\frac{2^{j+7}|\gS|mN_j^2}{\delta}
 \le L_j,
 \qquad
 L_j\coloneq\log\frac{128|\gS|mN_j^3}{\delta}.
\end{equation*}
Introduce the square-root and linear parts
\begin{equation*}
 S_j\coloneq
 \sqrt{\frac{cL_j}{c_\gM(1-\gamma)N_j}},
 \qquad
 U_j\coloneq\frac{cL_j}{N_j},
 \qquad R_j\coloneq S_j+U_j.
\end{equation*}
Then
\begin{equation}\label{eq:harmonic_block_envelope}
 b_j\le2S_j+\frac13U_j\le2R_j.
\end{equation}
Because $L_j/N_j$ decreases with $j$ (here $L_0\ge\log128>3$), both
$S_j$ and $U_j$ decrease. Increase $u$, and consequently $N$, once more
if necessary, until
\begin{equation*}
  2S_0+\frac13U_0\le
  \min\left\{\frac{r_{\mathrm{out}}}{16},
             \frac{\mu}{64L_h}\right\}.
\end{equation*}
All $b_j$ then satisfy the local-capture conditions. Every requirement above holds for all sufficiently large $u$:
all three entrance terms in Equation~\eqref{eq:harmonic_scale_sensitive_entrance_conditions}
tend to zero, and both local noise envelopes vanish. We may
therefore take the smallest integer $u\ge\max\{1,t_0\}$ satisfying
this finite list of deterministic inequalities. This makes the initial
time $N$ constructive and deterministic.

Set $E_j=\norm{\bm\theta^{(N_j)}-\bm\theta_m}_\infty$. If
$E_j\le r_{\mathrm{out}}/2$, the exact local-capture estimate on block $j$
gives, outside an event of conditional probability at most $\delta_j$,
\begin{equation*}
 E_{j+1}\le qE_j+b_j+\frac{32L_h}{15\mu}b_j^2
 \le qE_j+\frac{31}{30}b_j.
\end{equation*}
Since $q\le1/4$ and $b_j\le r_{\mathrm{out}}/16$, this bound also gives
$E_{j+1}<r_{\mathrm{out}}/2$. Thus every successive application of the
local lemma is legitimate on the preceding success events. Conditioning
on $\gF_{N_j}$ and summing over the blocks bounds their total failure
probability by $\sum_{j\ge0}\delta_j$; no independence is needed.
Iteration yields
\begin{equation}\label{eq:harmonic_endpoint_iteration}
 E_j\le q^jE_0+\frac{31}{30}
 \sum_{\ell=0}^{j-1}q^{j-1-\ell}b_\ell.
\end{equation}

For $0\le\ell<j$, monotonicity of $L_j$ gives
\begin{equation*}
 S_\ell\le2^{(j-\ell)/2}S_j,
 \qquad U_\ell\le2^{j-\ell}U_j.
\end{equation*}
The linear term is kept separate because it decays as $N_j^{-1}$,
whereas the square-root term decays as $N_j^{-1/2}$, up to logarithms.
With $r=j-1-\ell$ and $q\le1/4$, we obtain
\begin{equation*}
\begin{aligned}
 \sum_{\ell=0}^{j-1}q^{j-1-\ell}b_\ell
 &\le2\sqrt2S_j\sum_{r=0}^{j-1}(q\sqrt2)^r
      +\frac23U_j\sum_{r=0}^{j-1}(2q)^r\\
 &\le\frac{2\sqrt2}{1-\sqrt2/4}S_j+\frac43U_j.
\end{aligned}
\end{equation*}
Define the numerical constant
\begin{equation*}
 K_*\coloneq\frac{31\sqrt2}{15(1-\sqrt2/4)}.
\end{equation*}
Since $K_*>62/45$, Equation~\eqref{eq:harmonic_endpoint_iteration}
implies $E_j\le q^jE_0+K_*R_j$. The inherited error can eventually
be absorbed into $R_j$: indeed,
$R_j\ge2^{-j/2}S_0$ and hence
$q^jr_{\mathrm{out}}/(2R_j)\to0$. Let $j_0$ be the smallest
nonnegative integer for which
$q^{j_0}r_{\mathrm{out}}/2\le R_{j_0}$. This inequality remains true
thereafter because $R_{j+1}\ge R_j/2$ and $q\le1/4$. Consequently,
for every $j\ge j_0$,
\begin{equation}\label{eq:harmonic_endpoint_rate}
 E_j\le(1+K_*)R_j.
\end{equation}

Now fix an integer $T\ge N_{j_0}$ and choose $J$ such that
$N_J\le T<N_{J+1}$. If $T>N_J$, apply the local lemma on the partial
block $[N_J,T]$ with failure probability $\delta/4$. Since $T<2N_J$,
\begin{equation*}
 \log\frac{16|\gS|mT^2}{\delta}
 \le\log\frac{64|\gS|mN_J^2}{\delta}\le L_J,
\end{equation*}
so its envelope satisfies
$b_{\mathrm{part}}\le2S_J+U_J/3\le2R_J$ and the local-capture
smallness conditions. Using the exact bootstrap estimate without
requiring contraction on this partial block gives
\begin{equation*}
\begin{aligned}
 \norm{\bm\theta^{(T)}-\bm\theta_m}_\infty
 &\le E_J+b_{\mathrm{part}}
       +\frac{32L_h}{15\mu}b_{\mathrm{part}}^2\\
 &\le\left(1+K_*+\frac{31}{15}\right)R_J.
\end{aligned}
\end{equation*}
The same bound follows directly from
Equation~\eqref{eq:harmonic_endpoint_rate} when $T=N_J$.

Set
\begin{equation*}
 L_T\coloneq\log\frac{128|\gS|mT^3}{\delta},
 \qquad K_{\mathrm{harm}}\coloneq K_*+\frac{46}{15}.
\end{equation*}
Since $N_J\le T<2N_J$ and $L_J\le L_T$, the preceding bound implies
\begin{equation}
 \norm{\bm\theta^{(T)}-\bm\theta_m}_\infty
 \le K_{\mathrm{harm}}\left\{
 \sqrt{\frac{2cL_T}{c_\gM(1-\gamma)T}}
 +\frac{2cL_T}{T}\right\}.
 \label{eq:harmonic_hp_bound_explicit}
\end{equation}
For $C_{2,*}=6K_{\mathrm{harm}}$, we have
$L_T\le3\log(C_{2,*}|\gS|mT/\delta)$, and
Equation~\eqref{eq:harmonic_hp_bound_explicit} implies
Equation~\eqref{eq:harmonic_hp_bound} with $C=C_{2,*}$.

The entrance and initial-capture events cost at most $\delta/2$, all
full dyadic blocks cost at most $\delta/4$, and the final partial block
costs at most $\delta/4$. Thus the bound holds with probability at least
$1-\delta$ for each fixed $T\ge T_{\mathrm{harm}}(\delta)$, where
$T_{\mathrm{harm}}(\delta)=N_{j_0}$ is finite and deterministic.
Only the local contraction condition $c\ge C_{\mathrm{harm}}/\mu$
is required of the gain coefficient. 
\end{proof}

\paragraph{Dependence of the sufficient burn-in on quantile resolution.}
Under Assumption~\ref{assump:density}(ii), keep
$c_\gM,C_0,L,\gamma$ fixed and take the natural harmonic tuning
$c\asymp\mu^{-1}\asymp m$ together with $t_0\asymp c$.
Then
\begin{equation*}
 r_{\mathrm{out}}\asymp m^{-1},\qquad
 \mu\asymp m^{-1},\qquad
 \beta=\widetilde\Theta(m),\qquad
 B_\Phi=O(1),\qquad
 p=\lambda_{\mathrm h}c=\Theta(1),
\end{equation*}
with $p\ge4/3$ under the gain condition in
Corollary~\ref{cor:harmonic_stepsize}.

Recall that the scale-sensitive entrance construction takes
\begin{equation*}
 v=\left\lceil (u+t_0)^{4/3}-t_0\right\rceil,
 \qquad N=2v.
\end{equation*}
The three conditions in
Equation~\eqref{eq:harmonic_scale_sensitive_entrance_conditions}
can be satisfied, up to logarithmic factors, by taking
\begin{equation*}
 u\gtrsim m^3.
\end{equation*}
Indeed, the initial-error term is then already smaller than
$r_{\mathrm{out}}$, while the potential-curvature and martingale terms
require, respectively,
\begin{equation*}
 v\gtrsim m^4
 \qquad\text{and}\qquad
 v\gtrsim m^4
\end{equation*}
up to logarithmic factors.  The initial local-capture condition
$\bar b_u^{\mathrm{init}}
 \le\min\{r_{\mathrm{out}}/16,\mu/(64L_h)\}$
also requires only $u\gtrsim m^3$ up to logarithmic factors.
Consequently, entrance and initial capture can be completed by a
deterministic time
\begin{equation}
 N=\widetilde O(m^4).
 \label{eq:harmonic_initial_burnin_m}
\end{equation}

It remains to absorb the error inherited at time $N$ into the stochastic
scale of the subsequent dyadic blocks.  Recall that
$N_j=2^jN$ and
\begin{equation*}
 R_j=S_j+U_j,\qquad
 S_j=
 \sqrt{\frac{cL_j}{c_\gM(1-\gamma)N_j}},
 \qquad
 U_j=\frac{cL_j}{N_j}.
\end{equation*}
Under the preceding choice of $N$,
\begin{equation*}
 S_0=\widetilde\Omega(m^{-3/2}),
 \qquad
 R_j\ge 2^{-j/2}S_0
 =\widetilde\Omega\!\left(
   m^{-3/2}2^{-j/2}
 \right).
\end{equation*}
Since $E_0\le r_{\mathrm{out}}/2=O(m^{-1})$, the defining condition
$q^{j_0}r_{\mathrm{out}}/2\le R_{j_0}$ is satisfied once
\begin{equation*}
 (q\sqrt2)^{j_0}
 \lesssim m^{-1/2}
\end{equation*}
up to logarithmic factors.  Because $q\le1/4$, one may therefore take
\begin{equation*}
 j_0\le \frac13\log_2m+O(\log\log m).
\end{equation*}
Combining this with Equation~\eqref{eq:harmonic_initial_burnin_m} gives
the sufficient harmonic burn-in bound
\begin{equation}\label{eq:definition_T_harm}
 T_{\mathrm{harm}}(\delta)
 =N_{j_0}
 =\widetilde O(m^{13/3}),
\end{equation}
which is also polynomial in $m$. 
The same polynomial scaling holds under
Assumption~\ref{assump:density}(i) if
$\Delta_m\gtrsim m^{-1}$.  If the boundary margin is smaller, it also
enters $r_{\mathrm{out}}$ and the preceding scaling need not apply.

%% file: tex/B_technical_lemmas.tex
\subsection{Global boundedness}

\begin{lemma}\label{lem:pathwise_boundedness}
Let $B$ be the largest possible discounted return and let $\Delta$ be the
largest extra excursion caused by one update:
\begin{equation*}
 B\coloneq(1-\gamma)^{-1},\qquad
 \Delta\coloneq \bar\alpha(1-\gamma)^{-1}.
\end{equation*}
If $\bm\theta^{(0)}\in[0,B]^{\gS\times[m]}$, then, almost surely,
\begin{equation}\label{eq:pathwise_parameter_box}
 \bm\theta^{(t)}\in[-\Delta,B+\Delta]^{\gS\times[m]}
 \qquad\text{for every }t\ge0.
\end{equation}
Consequently,
\begin{equation}\label{eq:pathwise_error_bound}
 \norm{\bm\theta^{(t)}-\bm\theta_m}_\infty
 \le\frac{1+\bar\alpha}{1-\gamma}
 \qquad\text{for every }t\ge0.
\end{equation}
\end{lemma}

\begin{proof}[Proof of Lemma~\ref{lem:pathwise_boundedness}]
We prove Equation~\eqref{eq:pathwise_parameter_box} by induction.  Suppose
all coordinates of $\bm\theta^{(t)}$ belong to $[-\Delta,B+\Delta]$. If
$\theta^{(t)}(s,i)<-\gamma\Delta$, then every sampled quantity
$r^{(t+1,s)}+\gamma\theta^{(t)}(s^{\prime(t+1,s)},j)$ is at least
$-\gamma\Delta$, so the update of this coordinate is nonnegative. Otherwise,
since $\alpha_t\le\bar\alpha$,
\begin{equation*}
 \theta^{(t+1)}(s,i)\ge-\gamma\Delta-\bar\alpha=-\Delta.
\end{equation*}
Similarly, every such sampled quantity is at most
$1+\gamma(B+\Delta)=B+\gamma\Delta$. If
$\theta^{(t)}(s,i)>B+\gamma\Delta$, its update is nonpositive; otherwise,
\begin{equation*}
 \theta^{(t+1)}(s,i)\le B+\gamma\Delta+\bar\alpha=B+\Delta.
\end{equation*}
This proves the invariant box.  Since
$\bm\theta_m\in[0,B]^{\gS\times[m]}$, Equation~\eqref{eq:pathwise_error_bound}
follows.
\end{proof}

\subsection{Local Density Stability}

The following lemma supplies the density lower bound used in the
local linearization.

\begin{lemma}\label{lem:local_density_stability}
For every $\bm\theta$ satisfying
$\norm{\bm\theta-\bm\theta_m}_\infty\le r_{\mathrm{out}}$ and every
$z$ satisfying
$\abs{z-\theta_m(s,i)}\le r_{\mathrm{out}}$, we have
\begin{equation}\label{eq:local_density_lower}
 p_{(\gT^\pi\bm\eta_{\bm\theta})(s)}(z)
 \ge\frac{d_{s,i}}2
 \ge\frac{c_\gM}{2}\tau_i(1-\tau_i).
\end{equation}
\end{lemma}

\begin{proof}[Proof of Lemma~\ref{lem:local_density_stability}]
For every $z\in\RB$ and
$\bm\theta\in\RB^{\gS\times[m]}$,
\begin{equation}\label{eq:bellman_target_density_formula}
 p_{(\gT^\pi\bm\eta_{\bm\theta})(s)}(z)
 =\frac1m
 \sum_{a\in\gA,\,s^\prime\in\gS}\sum_{j=1}^m
 \pi(a\mid s)P(s^\prime\mid s,a)
 p_{s,a}\bigl(z-\gamma\theta(s^\prime,j)\bigr).
\end{equation}
If $\norm{\bm\theta-\bm\theta_m}_\infty\le r_{\mathrm{out}}$ and
$\abs{z-\theta_m(s,i)}\le r_{\mathrm{out}}$, then
\begin{equation}\label{eq:density_argument_perturbation}
\abs{
z-\gamma\theta(s^\prime,j)
-\prn{\theta_m(s,i)-\gamma\theta_m(s^\prime,j)}}\le(1+\gamma)r_{\mathrm{out}}.
\end{equation}
Under condition~{\rm (i)} of Assumption~\ref{assump:density}, the
definition of $r_{\mathrm{out}}$ ensures that the two
arguments in Equation~\eqref{eq:density_argument_perturbation}
belong to the same connected component of
$\RB\setminus\{0,1\}$. The density is $L$-Lipschitz on the interior
component and is zero on the two exterior components. Under
condition~{\rm (ii)}, $p_{s,a}$ is $L$-Lipschitz on all of
$\RB$. Thus, in both cases,
\begin{equation}\label{eq:local_density_perturbation_bound}
\begin{aligned}
&\abs{
p_{(\gT^\pi\bm\eta_{\bm\theta})(s)}(z)
-p_{(\gT^\pi\bm\eta_m)(s)}\bigl(\theta_m(s,i)\bigr)}\\
\le& L(1+\gamma)r_{\mathrm{out}}
\le\frac{c_\gM}{32m}
\le\frac{d_{s,i}}8.
\end{aligned}
\end{equation}
Here the last inequality follows by taking $u\to0$ in the definition
of $c_\gM$:
\begin{equation}\label{eq:technical_density_lower_from_cM}
 d_{s,i}\ge c_\gM\tau_i(1-\tau_i)
 \ge\frac{c_\gM}{4m}.
\end{equation}
Equations~\eqref{eq:local_density_perturbation_bound} and
\eqref{eq:technical_density_lower_from_cM} prove
Equation~\eqref{eq:local_density_lower}.
\end{proof}

\subsection{Local Concentration Tools}

\begin{lemma}\label{lem:uniform_linearized_martingale}
Suppose $\bar\alpha C_0\le1$. Let $n\le T$ be an
$(\gF_t)$-stopping time, and let
$\sigma\ge n$ be a stopping time such that
$\norm{\bm\theta^{(t)}-\bm\theta_m}_\infty\le r_{\mathrm{out}}$
on $\{t<\sigma\}$.  For $r\leq k$, set
\begin{equation*}
 \bPhi_{r,k}\coloneq\prod_{j=r}^{k-1}(\bI-\alpha_j\bG_m).
\end{equation*}
For every $\delta\in(0,1)$, conditionally on $\gF_n$, with probability
at least $1-\delta$,
\begin{equation*}
\begin{aligned}
&\max_{(s,i)\in\gS\times[m]}\max_{n\le r<k\le T}
\abs{\sum_{t=r}^{k-1}\alpha_t\bm e_{s,i}^\top
\bPhi_{t+1,k}\ind\{t<\sigma\}\bxi^{(t+1)}}\\
\leq&
2\sqrt{\frac{\alpha_n}{c_\gM(1-\gamma)}
\log\frac{2|\gS|mT^2}{\delta}}
+\frac{\alpha_n}{3}
\log\frac{2|\gS|mT^2}{\delta}.
\end{aligned}
\end{equation*}
\end{lemma}

\begin{proof}[Proof of Lemma~\ref{lem:uniform_linearized_martingale}]
Fix a coordinate index $q\in\gS\times[m]$ and times $r<k$. Let
$\bm e_q$ be the corresponding standard basis vector, and put
\begin{equation*}
 \bv_t^\top\coloneq\bm e_q^\top\bPhi_{t+1,k},\qquad
 u_t\coloneq\bv_t^\top\bm1,\qquad r\le t<k.
\end{equation*}
Here $u_t$ is the total mass of the nonnegative row vector $\bv_t$.
For every $j$, Lemma~\ref{lem:jacobian_summarize} gives
$\bI-\alpha_j\bG_m\ge\bm0$ and
$\norm{\bI-\alpha_j\bG_m}_\infty\le1-\mu\alpha_j\le1$.
It follows that $\bv_t\ge\bm0$ and $0\le u_t\le1$. Define its mass after one additional
linearized step by
\begin{equation*}
 u_t^\prime\coloneq\bv_t^\top(\bI-\alpha_t\bG_m)\bm1.
\end{equation*}
Equation~\eqref{eq:drift_matching} then yields
\begin{equation*}
 u_t-u_t^\prime
 =(1-\gamma)\alpha_t\bv_t^\top\bD_m\bm1\ge0.
\end{equation*}
All propagator factors are polynomials in the same matrix $\bG_m$ and hence
commute.  Consequently,
\begin{equation*}
 u_t^\prime=\bm e_q^\top\bPhi_{t,k}\bm1=u_{t-1}\qquad(r<t<k),
\end{equation*}
with the endpoint identity
$u_r^\prime=\bm e_q^\top\bPhi_{r,k}\bm1$.  Since
$0\le u_t^\prime\le u_t\le1$, we obtain
\begin{equation*}
 \sum_{t=r}^{k-1}u_t(u_t-u_t^\prime)
 \le\sum_{t=r}^{k-1}(u_t-u_t^\prime)
 =u_{k-1}-u_r^\prime\le1.
\end{equation*}
Moreover, $\ind\{t<\sigma\}$ is $\gF_t$-measurable, so stopping preserves
the martingale-difference property.  By
Equation~\eqref{eq:variance_matching} and monotonicity of the step sizes,
\begin{equation*}
\begin{aligned}
&\sum_{t=r}^{k-1}\alpha_t^2
\EB\left[
\bigl(\bv_t^\top\ind\{t<\sigma\}\bxi^{(t+1)}\bigr)^2
\mid\gF_t\right]\\
\leq&\frac{2\alpha_r}{c_\gM(1-\gamma)}
\sum_{t=r}^{k-1}u_t(u_t-u_t^\prime)
\le\frac{2\alpha_r}{c_\gM(1-\gamma)}.
\end{aligned}
\end{equation*}
Moreover,
\begin{equation*}
 \left|\alpha_t\bv_t^\top\ind\{t<\sigma\}\bxi^{(t+1)}\right|
 \le \alpha_t\bv_t^\top\bm1\le\alpha_r.
\end{equation*}
Freedman's inequality and a union bound over at most $|\gS|mT^2$ tuples
$(s,i,r,k)$ complete the proof.
\end{proof}

\begin{lemma}\label{lem:nonlinear_epoch_bootstrap}
Use the local contraction scale $\mu$, quadratic-remainder coefficient $L_h$,
and propagator $\bPhi_{r,k}$ introduced in
Section~\ref{sec:proof_variance_sensitive_local}. Let $n$ be an
$(\gF_t)$-stopping time satisfying
\begin{equation*}
    \norm{\bm\theta^{(n)}-\bm\theta_m}_\infty\le \frac{r_{\mathrm{out}}}{2},
\end{equation*}
and define
\begin{equation*}
    \sigma\coloneq\inf\left\{t\ge n:\norm{\bm\theta^{(t)}-\bm\theta_m}_\infty>r_{\mathrm{out}}\right\},
\end{equation*}
with the convention $\inf\varnothing=\infty$.

Suppose that on an event $\gE$, the propagated linear noise is uniformly
bounded by a deterministic envelope $b$:
\begin{equation*}
    \max_{n\le r<k\le T}\left\|\sum_{t=r}^{k-1}\alpha_t\bPhi_{t+1,k}\ind\{t<\sigma\}\bxi^{(t+1)}\right\|_\infty\le b,
\end{equation*}
and
\begin{equation}\label{eq:abstract_bootstrap_smallness}
    b\le \frac{r_{\mathrm{out}}}{16},\ \frac{L_h r_{\mathrm{out}}}{\mu}\le\frac1{16},\ b\le \frac{\mu}{64L_h},
\end{equation}
where the last bound is interpreted as vacuous when $L_h=0$.
Then on $\gE$, we have $\sigma>T$ and
\begin{equation*}
    \norm{\bm\theta^{(T)}-\bm\theta_m}_\infty
    \le \norm{\bm\theta^{(n)}-\bm\theta_m}_\infty
    \exp\left\{-\frac{15\mu}{16}\sum_{t=n}^{T-1}\alpha_t\right\}
    +b+\frac{32L_h}{15\mu}b^2.
\end{equation*}
\end{lemma}

\begin{proof}[Proof of Lemma~\ref{lem:nonlinear_epoch_bootstrap}]
First define the auxiliary sequence $\bw^{(k)}$, which contains only the
propagated stochastic perturbations, with the same sign as in the error
recursion. For $n\le k\le T$, let
\begin{equation*}
    \bw^{(n)}\coloneq\bm0,\qquad
    \bw^{(k)}\coloneq-\sum_{t=n}^{k-1}\alpha_t
    \bPhi_{t+1,k}\ind\{t<\sigma\}\bxi^{(t+1)}.
\end{equation*}
Then we know that
\begin{equation*}
    \norm{\bw^{(k)}}_\infty\le b
\end{equation*}
and
\begin{equation*}
    \bw^{(t+1)}=(\bI-\alpha_t\bG_m)\bw^{(t)}
    -\alpha_t\ind\{t<\sigma\}\bxi^{(t+1)}.
\end{equation*}

The residual $y_t\coloneq\norm{\be^{(t)}-\bw^{(t)}}_\infty$ therefore tracks
the deterministic and nonlinear parts of the error.  For
$n\le t<T\wedge\sigma$, subtracting the two recursions gives
$\be^{(t+1)}-\bw^{(t+1)}
=(\bI-\alpha_t\bG_m)(\be^{(t)}-\bw^{(t)})
-\alpha_t\bR(\be^{(t)})$, and hence
\begin{equation}\label{eq:bootstrap_scalar_recursion}
    y_{t+1}
    \le(1-\mu\alpha_t)y_t
    +\alpha_t\norm{\bR(\be^{(t)})}_\infty
    \le(1-\mu\alpha_t)y_t
    +L_h\alpha_t(y_t+b)^2.
\end{equation}

Now we claim that $y_k\le r_{\mathrm{out}}/2$ for every
$n\le k\le T\wedge\sigma$. Suppose $y_t\le r_{\mathrm{out}}/2$. Then
\begin{equation*}
\begin{aligned}
    L_h(y_t+b)^2
    &\le2L_hy_t^2+2L_hb^2\\
    &\le L_hr_{\mathrm{out}}y_t+2L_hb^2
    \le\frac{\mu}{16}y_t+2L_hb^2.
\end{aligned}
\end{equation*}
Moreover, if $L_h>0$, the two bounds on $b$ in
Equation~\eqref{eq:abstract_bootstrap_smallness} imply
\begin{equation*}
    2L_hb^2\le\frac{\mu b}{32}
    \le\frac{\mu r_{\mathrm{out}}}{512},
\end{equation*}
while the same bound is immediate when $L_h=0$. Because
$0\le\mu\alpha_t\le1$, the coefficient
$1-15\mu\alpha_t/16$ is nonnegative. Thus,
Equation~\eqref{eq:bootstrap_scalar_recursion} and $y_t\le r_{\mathrm{out}}/2$ yield
\begin{equation}\label{eq:bootstrap_contractive_recursion}
\begin{aligned}
    y_{t+1}
    &\le\left(1-\frac{15}{16}\mu\alpha_t\right)y_t
    +2L_h\alpha_tb^2\\
    &\le\left(1-\frac{15}{16}\mu\alpha_t\right)
    \frac{r_{\mathrm{out}}}{2}+\frac{\mu\alpha_t r_{\mathrm{out}}}{512}\\
    &=\frac{r_{\mathrm{out}}}{2}
    -\frac{239}{512}\mu\alpha_t r_{\mathrm{out}}
    \le\frac{r_{\mathrm{out}}}{2}.
\end{aligned}
\end{equation}
Because $y_n=\norm{\be^{(n)}}_\infty\le r_{\mathrm{out}}/2$, induction proves the
claim.

Therefore we know that for every $n\le k\le T\wedge\sigma$,
\begin{equation*}
    \norm{\be^{(k)}}_\infty
    \le y_k+\norm{\bw^{(k)}}_\infty
    \le\frac{r_{\mathrm{out}}}{2}+b
    \le\frac{9r_{\mathrm{out}}}{16}<r_{\mathrm{out}}.
\end{equation*}
If $\sigma\le T$, this estimate at $k=\sigma$ contradicts the
definition of $\sigma$ and thus $\sigma>T$. 

Now the first inequality in
Equation~\eqref{eq:bootstrap_contractive_recursion} holds on the whole
interval. Moreover we have
\begin{equation*}
\sum_{t=n}^{T-1}\alpha_t
\prod_{\ell=t+1}^{T-1}
\left(1-\frac{15}{16}\mu\alpha_\ell\right)=
\frac{1-\prod_{t=n}^{T-1}
\left(1-\frac{15}{16}\mu\alpha_t\right)}{(15/16)\mu}
\le\frac{16}{15\mu}.
\end{equation*}
Iterating the recursion of $y_t$ gives
\begin{equation*}
\begin{aligned}
    y_T
    &\le y_n\prod_{t=n}^{T-1}\left(1-\frac{15}{16}\mu\alpha_t\right)+2L_hb^2\sum_{t=n}^{T-1}\alpha_t\prod_{\ell=t+1}^{T-1}\left(1-\frac{15}{16}\mu\alpha_\ell\right)\\
    &\le\norm{\be^{(n)}}_\infty
    \exp\left\{-\frac{15\mu}{16}\sum_{t=n}^{T-1}\alpha_t\right\}
    +\frac{32L_h}{15\mu}b^2.
\end{aligned}
\end{equation*}
Therefore we conclude that
\begin{equation*}
    \norm{\be^{(T)}}_\infty\le y_T+b
    \leq\norm{\be^{(n)}}_\infty
    \exp\left\{-\frac{15\mu}{16}\sum_{t=n}^{T-1}\alpha_t\right\}
    +b+\frac{32L_h}{15\mu}b^2,
\end{equation*}
which completes the proof.
\end{proof}

%% file: tex/D_additional_simulations.tex
\section{Additional Numerical Results}
\label{sec:additional_simulations}

\begin{figure}[H]
  \centering
  \includegraphics[width=\textwidth]{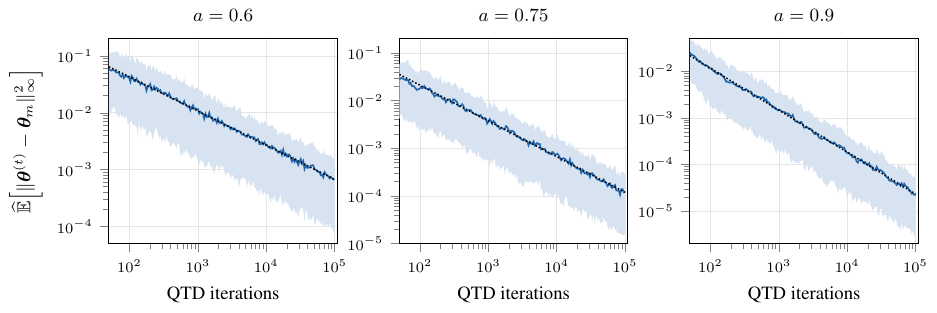}
  \caption{Squared last-iterate error for three polynomial exponents over 200
  independent QTD trajectories.  Blue lines and shading show the mean and
  10th--90th percentiles; dotted lines have slopes fixed at $-a$.}
  \label{fig:qtd_rate_decay_sweep}
\end{figure}

\paragraph{Complete protocol for Figure~\ref{fig:qtd_simulation_summary}.}
Panels~(a)--(b) use a one-state, one-action MDP with deterministic self
transition, $\gamma=0.5$, and independent
$\operatorname{Unif}[0,1]$ rewards across iterations and replications.  We
use $m=7$ quantiles at levels $\tau_i=(2i-1)/14$, initialize all quantile
locations at zero, and run 200 trajectories through $T=10^5$.  The
same reward sample is shared by the seven coordinates within one QTD update,
as specified in Section~\ref{Section:preliminary}.  We compute the benchmark
$\bm\theta_m$ independently by iterating the quantile-projected Bellman
operator and solving each quantile by bisection; its maximum CDF residual is
$3.4\times10^{-16}$.

Panel~(a) uses $\alpha_t=4/(t+20)^{0.75}$.  Its dotted line has slope fixed
at $-0.75$ and only its intercept is fitted over
$t\in[5\times10^3,10^5]$.  Panel~(b) compares
$\alpha_t=0.05$, $4/(t+20)^{0.75}$, and $20/(t+100)$.  The base seed is
20260901, with an offset of 1000003 between schedules; Panel~(a) and the
polynomial curve in Panel~(b) share the same trajectory set.  Curves are
empirical means of the squared sup-norm error at 181 logarithmically spaced
checkpoints.  Panel~(a) additionally shows the empirical 10th--90th
percentile band.

For Panel~(c), define
$\tau=\inf\{t\ge0:\norm{\bm\theta^{(t)}-\bm\theta_m}_\infty
\le r_{\mathrm{out}}/2\}$.  We use the same singleton MDP construction with
one quantile, $\gamma=0.01$, and $\operatorname{Unif}[0,1]$ rewards.  In this
case, $\theta_m=0.5/(1-\gamma)=0.505051$, and we initialize at
$\theta^{(0)}=\theta_m+r_{\mathrm{out}}$.  The constants computed from the
fixed-$m$ refinement in Equation~\eqref{eq:fixed_m_identifiability} are
$c_{\gM,1}=2(1-\gamma)=1.98$, $r_{\mathrm{out}}=0.123762$,
$c_{\mathrm g}=0.007581$, and $\beta=182.503$.  We use
$\alpha=0.75c_{\mathrm g}/\beta=3.116\times10^{-5}$ and seed 20260902.
Among 5,000 independent trajectories, none is censored at horizon
$2\times10^5$: the median entrance time is 22,429, the empirical 90th
percentile is 23,577, and the maximum is 25,945.

For this deterministic initialization,
$D_{0,\beta}=\Phi_\beta(\theta^{(0)}-\theta_m)-r_{\mathrm{out}}/2
=0.061881$.  Panel~(c) plots the conditional bound in
Equation~\eqref{eq:initialization_specific_entrance} with this exact realized
budget, rather than replacing it by the coarser initialization-uniform value
$D_{\mathrm g}=0.952049$.  The empirical survival function is evaluated from
the 5,000 hitting times.  Zero empirical probabilities are shown at
$10^{-4}$ solely for display on the logarithmic axis; the accompanying CSV
retains the exact zero values and the unclipped theoretical probabilities.

\paragraph{Additional polynomial exponents.}
Figure~\ref{fig:qtd_rate_decay_sweep} repeats the fixed-slope diagnostic from
Figure~\ref{fig:qtd_simulation_summary}(a) for
$a\in\{0.6,0.75,0.9\}$.  All other settings are unchanged.  Each dotted
reference has slope fixed at $-a$ and is fitted only in intercept over
$t\in[5\times10^3,10^5]$.  Across all three exponents, the empirical curves
are nearly parallel to their theoretical references after burn-in.